\documentclass{article}

\usepackage[margin=1in]{geometry}
\usepackage[colorinlistoftodos]{todonotes}
\usepackage{titling,kantlipsum}
\usepackage{amsfonts,amsthm,amsmath,amssymb}
\usepackage{xcolor}
\usepackage{graphicx}
\usepackage{subcaption}
\usepackage{enumitem}
\usepackage{mdframed}
\usepackage{authblk}
\usepackage{parskip}
\usepackage{bbm}
\usepackage{algorithm}
\usepackage[noend]{algpseudocode}
\algrenewcommand\algorithmiccomment[1]{\hfill \textcolor{gray}{// #1}}
\usepackage{setspace}
\usepackage[colorlinks, linkcolor=black,citecolor=black,urlcolor = black, bookmarks=true]{hyperref}
\usepackage[nameinlink, capitalize]{cleveref}
\usepackage{ragged2e}
\usepackage{tcolorbox}
\hypersetup{
	colorlinks,
	linkcolor={red!75!black},
	citecolor={blue!75!black},
	urlcolor={red!75!black},
}
\allowdisplaybreaks
\usetikzlibrary{decorations.pathreplacing}

\newtheorem{theorem}{Theorem}

\newtheorem{lemma}{Lemma}
\newtheorem{definition}{Definition}

\newtheorem{example}{Example}

\newtheorem{conjecture}{Conjecture}

\DeclareMathOperator*{\Uniform}{Unif}
\DeclareMathOperator*{\Bern}{Bern}
\DeclareMathOperator*{\Var}{Var}

\newcommand*\calF{\mathcal{F}}

\newcommand*\calA{\mathcal{A}} 
\newcommand*\calL{\mathcal{L}}

\newcommand*\calP{\mathcal{P}}
\newcommand*\calE{\mathcal{E}}
\newcommand*\calB{\mathcal{B}}
\newcommand*\bbR{\mathbb{R}}
\newcommand*\bbE{\mathbb{E}}

\newcommand*\bbN{\mathbb{N}}

\newcommand*\ind{\mathbbm{1}}

\newcommand*\pfx{\mathrm{pfx}}
\newcommand*\el{{}}
\newcommand*\up{up}
\newcommand*\low{low}
\newcommand*\decision{\mathrm{decision}}
\newcommand*\choice{\mathrm{choice}}

\makeatletter

\newcommand{\abstracttext}[1]{\def\@abstract{#1}}

\renewcommand{\maketitle}{
    \begingroup
    \renewcommand{\thefootnote}{\fnsymbol{footnote}} 

  \begin{center}
        {\LARGE \bfseries \@title \par}
        \vspace{2em}

        \begin{tabular*}{.8\textwidth}{@{\extracolsep{\fill}} c c}
            \shortstack{\textbf{Atul Ganju\footnotemark[2]}\\ganju@usc.edu\\University of Southern California} &
            \shortstack{\textbf{Travis McVoy\footnotemark[2]}\\mcvoy@usc.edu\\University of Southern California} \\
            \\\\
            \shortstack{\textbf{Shaddin Dughmi\footnotemark[3]}\\shaddin@usc.edu\\University of Southern California} &
            \shortstack{\textbf{Shang-Hua Teng}\\shanghua@usc.edu\\University of Southern California}
        \end{tabular*}

        \vspace{2em}
        {\@date \par}
    \end{center}

    \footnotetext[2]{Supported by a Viterbi School of Engineering Graduate Student Fellowship}
    \footnotetext[3]{Supported by NSF Grant CCF-2432219. Part of this work was done while the author was on sabbatical as the Carter and Tania Neild visiting professor at Northwestern University, as well as a visiting professor in the Data Science Institute at the University of Chicago}
    \endgroup 

    \vspace{1em}
    \begin{center}
        \begin{minipage}{0.85\textwidth}
            \small
            \setlength{\parindent}{1.5em}
            \justifying
            \@abstract
        \end{minipage}
    \end{center}
    \vspace{2em}
}
\makeatother

\title{A Theory of Time-Sensitive Language Generation: \\ Sparse Hallucination Beats Mode Collapse}

\date{\today}
\abstracttext{\begin{abstract}
    We study language generation in the limit under a global preference ordering on strings, as introduced by Kleinberg and Wei. As in \cite{KW, KW2}, we aim for \emph{breadth}, but impose an additional requirement of timeliness: higher-ranked strings should be generated earlier. A string is then only credited if it is generated before a deadline, where its deadline is defined by a function that maps a string’s rank in the target language to the time by which it must be produced. This is in keeping with a central consideration in machine learning, where inductive bias favors  ``simpler'' or ``more plausible'' outputs, all else being equal. We show that timely generation is impossible in a strong sense for eventually consistent generators—the protagonists of most prior related work. Under what is perhaps the mildest natural relaxation of consistency, a hallucination rate that vanishes over time, we show that we can circumvent our impossibility result. In particular, we can achieve optimal density  with respect to any superlinear deadline function. We also show this is tight by ruling out timely generation with linear deadlines and vanishing hallucination rate.
\end{abstract}}

\begin{document}
\maketitle
\setcounter{footnote}{0}
\renewcommand{\thefootnote}{\arabic{footnote}}

\crefname{figure}{Figure}{figures}

\section{Introduction}

Generative language models are a key building block of modern AI systems, enabling a multitude of downstream applications such as general-purpose dialog, creative writing, mathematical reasoning, and code synthesis \cite{towardsAutoMath, controlStyle, escapeBench}. 
Alongside this undeniable practical impact, a growing body of theoretical work seeks to understand the  principles and trade-offs governing language generation. 

An increasingly influential approach, often referred to as \emph{language generation in the limit}, was spurred by the seminal work of Kleinberg and Mullainathan~\cite{KM}, building on classical work concerned with language identification by Gold~\cite{Gold} and Angluin~\cite{Angluin}. This approach is characterized by its minimal assumptions, abstracting away many of the details of practical implementations and environments.
At a high level, language generation is modeled as a game played over time between the generator and an adversary: At the outset the adversary chooses a \emph{target language} from some countable collection of permissible languages, then the adversary and generator take turns generating strings. The adversary always generates a string from the target language, and the generator's goal is to eventually do the same. By focusing on the ``bare-bones'' essentials of language generation and pursuing guarantees only in the limit over time, this approach enables theoretical analysis of what is possible \emph{in principle}, including identifying and quantifying the intrinsic \emph{tension} between various design desiderata.

Much of the recent literature has focused on the tension between two desiderata: avoiding \emph{hallucination} by the generator, often formalized through notions such as \emph{consistency}, and achieving \emph{breadth} in the generator's outputs over the target language. We introduce a third concern in this paper which we call \emph{timeliness}, and study its interplay with hallucination and breadth. Timeliness postulates that ``more plausible'' or ``simpler'' outputs should be generated preferentially, which we take to mean earlier in the language generation game. Plausibility can represent any of the myriad forms of inductive bias employed in machine learning: preference for short programs in coding tasks, simpler or human-checkable proofs in mathematical reasoning tasks (judged, perhaps, by a trained reward model),  or distributional plausibility in dialog or creative writing tasks \cite{palPaper, mauve, cot22}.   We measure timeliness with respect to a \emph{deadline} function, mapping a string's position in the plausibility order to a time in the game by which it must be generated, lest the generator forego the associated credit.

In addition to being compelling in its own right, timeliness shares technical and conceptual connections with recent notions of breadth introduced by Kleinberg and Wei~\cite{KW}.  As background, there are several natural ways of modeling breadth, most of which lead to strong negative results---See \cite{kalavasis1, kalavasis2} for examples. This paints the theory into a corner by suggesting the inevitability of \emph{mode collapse}---the eventual deterioration of the generator to a tiny sliver of the target language. While mode collapse is a very real concern in practice~\cite{GANS2, GANS1}, such a stark prediction is perhaps increasingly out of step with empirical reality~\cite{noModeCollapseBoi}. A key conceptual insight came with the work of \cite{KW}, who measure breadth with respect to a preference ordering on strings. In particular, a generator achieves breadth in the target language if, when strings in the language are sorted in order of preference, its outputs are \emph{dense} in every sufficiently long prefix of the language. This approach to modeling breadth with respect to an inductive-bias on strings enabled remarkable progress in \cite{KW,KW2}, culminating in an algorithm achieving both consistency and optimal breadth in the limit. Timeliness can best be understood as a strengthening of the breadth requirements of \cite{KW,KW2}: in each prefix of the target language, a generated string only ``counts'' towards density if it is generated before its deadline.

Our technical work starts with the realization that the algorithms of \cite{KW,KW2}, and in fact any eventually consistent generator, can not achieve timely generation. This impossibility holds for any deadline function, and any nontrivial guarantee on the prefix-density of timely outputs.  In other words, consistent generators inevitably suffer \emph{mode collapse under a deadline}. This motivates relaxing consistency: by permitting hallucinations at a rate tending to zero over time, we recover timely generation in the limit for any superlinear deadline function, and provide essentially-matching impossibility results. Next, we present the relevant technical preliminaries before formally stating our results. We defer more thorough discussion of related work to \cref{app:related}. 

\subsection{Technical Preliminaries}

\textbf{Basic Notation.}
We often make use of countably infinite sets which are ordered (i.e.~enumerated) by a surjective map from the natural numbers. We denote such an \emph{ordered set} as $S=(s_1,s_2,\ldots)$, where $s_i$ denotes the $i$th element of the enumeration. We also use $S_i = (s_1,\ldots,s_i) = s_{1:i}$ to denote the first $i$ elements of the enumeration of $S$, and often refer to it as the prefix of $S$ of size $i$. Alternatively, we sometimes do not fix the ordering of a countably infinite set $K$, and instead consider the space of all enumerations which we denote by $\calE(K)$. Note that our notion of enumeration here can allow repetition, though in some cases we restrict attention to bijective enumerations. 

\textbf{The Language Generation Game.} 
As in \cite{KM,KW}, we represent language generation as a game between an adversary and a generator. Both players have access to a countably infinite collection of languages $\calL = (L^1, L^2, \ldots)$, defined over a countably infinite \emph{universe} $U = (u_1, u_2, \ldots)$ of ``strings.'' Note that both $\calL$ and $U$ are equipped with a bijective enumeration order, which is fixed and known. In the case of $U$, we think of the order $(u_1,u_2,\ldots)$ as encoding a preference (or inductive bias) on strings. This order on $U$ induces a preference order on the strings in each language $L \in \calL$ by restriction---i.e, strings in $L$ are also enumerated by a total order which is consistent with the order on $U$.

The adversary first selects a target language $K \in \calL$ and an enumeration  $E = (e_1, e_2, \ldots)$, where $E \in \calE(K)$. The generator $\calA = \{\calA_t : U^t \to U\}_{t=1}^\infty$ observes $E$ sequentially. At time $t$, the generator outputs  $o_t = \calA_t(e_{1:t}) \in U$ that has not previously been observed or output. Let $\calA(E) = (o_1, o_2, \ldots)$ denote the output sequence of the generator. We allow $\calA$ to be randomized, in which case each $\calA_t(e_{1:t})$ is a random variable, where expectations are over the internal randomness of the algorithm.

\textbf{Hallucination and Consistency.} We say the generator \emph{hallucinates} at time $t$ if $o_t \not\in K$. We say it is  \emph{eventually consistent} if it is guaranteed to eventually stops hallucinating; i.e., there exists some finite time $T$ (perhaps depending on $\calL$, $K$, and $E$)  such that for all $t\geq T$ the output $o_t\in K$. The celebrated result of \cite{KM} exhibits an eventually consistent generator for every language collection~$\calL$. 

Our impossibility results will necessitate relaxing consistency by allowing hallucinations that vanish over time. We define the \emph{hallucination rate} of $\calA$ at step $t$, as a function of $(\calL,K,E)$, as follows
\[    
    H_{\calA}^{\calL,K,E}(t) := \bbE\!\left[\frac{|\calA(E)_t\setminus K|}{t} \right],
\]
where the expectation is over the internal randomness of $\calA$. We say that $\calA$ has vanishing hallucination rate if, for every countable collection $\calL$, we have
\[
    \sup_{K\in\calL}\sup_{E\in\calE(K)}\lim_{t\to\infty}H_{\calA}^{\calL,K,E}(t)=0.
\]

\textbf{Breadth and Timeliness.} 
For a deadline function $D : \bbN \to \bbN$ and two sequences $S=(s_1,s_2,\ldots)$ and $R=(r_1,r_2,\ldots)$, we define the \emph{timely element-wise density of $S$ in $R$ under deadline $D$} at index $i$ as follows
\begin{align*}
    \mu^{\el}_i(S,R,D) = \frac{\sum_{j=1}^i  \ind\{r_j \in S_{D(j)}\}}{i}.
\end{align*}
The generator’s objective is to maximize the (lower) density of timely outputs, in the limit over prefixes of the target language. Formally, for any algorithm $\calA$ and deadline function $D$, we define its timely lower density, taken in the limit for a worst-case instance,  as follows.
\begin{align*}
    \mu^{\el}_{\low}(\calA, D)
    &= \inf\left\{\inf_{K\in \calL} \inf_{E\in \calE(K)} \liminf_{i \to \infty} \mu^{\el}_i(\calA(E), K, D): \calL \text{ is a countable language collection}\right\}.
\end{align*}
We note that the lower density employed in \cite{KW,KW2} corresponds to employing the trivial deadline function $D(i)=\infty$. There, the remarkable result of \cite{KW2} exhibits a generator which guarantees a lower density of $1/2$, which can easily be seen to be the best possible.\footnote{This is due to the turn-taking nature of the language generation game, whereby an aggressive adversary can ``beat'' the generator to half the strings in the language. This holds even for a trivial collection consisting of a single language.} In what follows, we will frequently use the generalized inverse of the deadline function, defined by $D^{-1}(t) := \min\{n\in\bbN : D(n)\ge t\}$. Intuitively, $D^{-1}(t)$ is the smallest index whose deadline has not yet passed by time $t$.

\subsection{Overview of Results}
Our first result (\cref{thm:finite-hallucination-impossibility}) exhibits a language family where every eventually-consistent generator---such as those in \cite{KW,KW2} and much of the prior work---must suffer a lower density of $0$ in the limit. This motivates a minimal relaxation of consistency, by allowing a vanishing hallucination rate over time. Even with this relaxation, we show another impossibility result (\cref{thm:linear-deadline-impossibility}): Nontrivial timeliness under linear deadlines is also impossible, again suffering a lower density of $0$ in the limit. 

These two impossibilities  set the stage for our main positive result (\cref{thm:lower-density-guarantee}): For any super-linear deadline function, we exhibit a randomized generation algorithm guaranteeing a timely density of $1/2$ with vanishing hallucination rate.  Our guarantees are optimal in several respects: the density guarantee of $1/2$ is the best possible as in \cite{KW2}, superlinear deadlines and some hallucination are necessary by our impossibility results, and furthermore we show quantitatively that our hallucination rate is essentially tight (\cref{lem:deadline-hallucination-impossibility}). Our proof is via a density-preserving black-box reduction to non-timely generation with breadth, into which we plug in the optimal generator of~$\cite{KW2}$. 

Finally, we discuss a more permissive upper-density notion in \cref{sec:upper-density} and show in \cref{app:positive-results} that our results are robust by extending them to general data models (e.g. partial enumeration and contamination models).
\section{Fundamental Limits for Consistent, Dense, Timely Generation}\label{sec:impossibility}

In this section, we establish that
{\em no} eventually consistent generator can achieve nontrivial lower density in a timely fashion under \emph{any} deadline function.

\begin{theorem}[Impossibility for Eventually Consistent Generators]\label{thm:finite-hallucination-impossibility}
    There exists a language family $\calL$ such that for every deadline function $D:\bbN\to\bbN$ and every eventually consistent generator $\calA$,
    \[
        \inf_{K\in\calL}\inf_{E\in\calE(K)}
        \liminf_{i\to\infty}
        \mu^{\el}_{i}(\calA(E),K,D) = 0.
    \]
    It follows that $\mu^{\el}_{\low}(\calA,D)=0$ for every deadline function $D$ and eventually-consistent generator $\calA$.
\end{theorem}
This rules out timeliness for the eventually-consistent generators of~\cite{KW,KW2}. More broadly, it shows that eventual consistency itself precludes nontrivial timely density, thus necessitating its relaxation. 

Our next theorem shows that even after relaxing consistency to allow for a vanishing rate of hallucination, no generator can achieve nontrivial lower density under linear deadlines.
\begin{theorem}[Impossibility for Linear-Time Generators]\label{thm:linear-deadline-impossibility}
    There exists a language family $\calL$ such that for any deadline function $D(i)=O(i)$, every generator $\calA$ with vanishing hallucination rate satisfies
    \[
        \inf_{K\in\calL}\inf_{E\in\calE(K)}
        \liminf_{i\to\infty}
        \mu^{\el}_{i}(\calA(E),K,D) = 0.
    \]
    It follows that $\mu^{\el}_{\low}(\calA,D)=0$ whenever  $D(i)=O(i)$ and $\calA$ has a vanishing hallucination rate.
\end{theorem}

This motivates designing timely generators with respect to super-linear deadlines, as we do \cref{sec:positive_result}. 

Both \cref{thm:finite-hallucination-impossibility} and \cref{thm:linear-deadline-impossibility} 
are consequences of the following fundamental lemma, which isolates the necessity of ``gradually'' vanishing hallucination in obtaining a nontrivial lower density.
\begin{lemma}[Hallucination Barrier]\label{lem:deadline-hallucination-impossibility}
   For any monotonically non-decreasing deadline function \smash{$D:\bbN\to\bbN$},  there exists a language family $\calL$ such that for any randomized generator $\calA$, if 
    \[
        \sup_{K\in\calL}\sup_{E\in\calE(K)}H_{\calA}^{\calL,K,E}(t)
        =
        o\!\left(\frac{D^{-1}(t)}{t}\right),
    \] 
    then $\mu^{\el}_{\low}(\calA,D)=0$ almost surely.
\end{lemma}

This result shows that any algorithm $\calA$ whose hallucination rate is $o(D^{-1}(t)/t)$ by time $t$ must suffer $\mu^{\el}_{\low}(\calA, D) = 0$ with respect to any deadline function $D$. In fact, it establishes something stronger: there exists a single language collection that serves as a universal witness against all such algorithms.

\begin{proof}[Proof of \cref{thm:finite-hallucination-impossibility} and \cref{thm:linear-deadline-impossibility}]
The first theorem follows by observing that eventual consistency implies the condition of \cref{lem:deadline-hallucination-impossibility} for any deadline function. The second theorem follows by applying the lemma to linear deadline functions.
\end{proof}

\subsection{Discussion of \cref{lem:deadline-hallucination-impossibility}: Hallucination is Necessary for Dense Generation}

We identify a structural property of language families such that any collection containing a subfamily with this property witnesses this impossibility. This structure is a nested sequence of languages that increasingly approximate a limiting language $L^\infty$, while ensuring that each intermediate language becomes arbitrarily sparse within $L^\infty$. This structure forces any generator with vanishing hallucination rate on each of the intermediate languages to achieve zero lower density on the limiting language.

\begin{definition}[Measure-Zero Chain]\label{def:nested family}
    A collection $\calL_{chain} = \{L^\infty\} \cup \bigcup_{j\geq 1} \{L^j\}$ forms a measure-zero chain if $L^1 \subseteq L^2 \subseteq \cdots \subseteq L^\infty$ is a nested sequence of languages for which:
    \[ 
        \bigcup\nolimits_{j\geq 1} L^j = L^\infty \qquad \text{and} \qquad \liminf_{i\to \infty} \mu_i(L^j, L^\infty, \infty) = 0 \qquad \forall j \in \bbN.
    \]
\end{definition}

This structure allows the adversary to present the generator with an enumeration of $L^\infty$ in such a way that, for long intervals of time, the sequence is indistinguishable from an enumeration of some intermediate language $L^j$, before eventually revealing additional elements of the true target. Since the generator must keep hallucinations infrequent with respect to $L^j$, it is forced to behave conservatively and output at most $o(D^{-1}(t))$ elements outside of $L^j$ by time $t$. However, because each $L^j$ captures only a vanishing fraction of $L^\infty$ along suitable prefixes, this conservative behavior prevents the generator from achieving nontrivial timely density in $L^\infty$.

Such families arise naturally in the study of language generation. For instance, the language family in Example~7 of \cite{KW} satisfies this property. For completeness, we provide a proof of this in \cref{app:examples} along with several other examples of language families that have this structure. 

Below we give a proof sketch of \cref{lem:deadline-hallucination-impossibility}. A full proof can be found in \cref{app:impossibility}.

\begin{proof}[Proof Sketch]
    The adversary constructs a total enumeration of $L^\infty$ in stages. At stage $j$, the sequence is consistent with an enumeration of $L^j$ up to time $t_j$ for a suitably large time $t_j$, before proceeding to stage $j+1$. The time $t_j$ is chosen large enough that (i) the expected number of hallucinations of the generator on $L^j$ by time $t_j$ is sufficiently small, and (ii) the density of $L^j$ in the prefix \smash{$L^\infty_{i_j}$} is sufficiently small, where $i_j := D^{-1}(t_j)$. The first condition can be realized because the generator has to have a vanishing hallucination rate in the case the target is $L^j$, while the second follows from the fact that $L^j$ occupies only a vanishing fraction of $L^\infty$ along ordered prefixes.

    Since the adversary’s constructed sequence agrees with an enumeration of $L^j$ up to time $t_j$, the generator produces identical outputs in both possible settings up to this time. In particular, if the generator avoids hallucinating on $L^j$, then almost all of its outputs up to time \smash{$t_j$} must lie in $L^j$. Accordingly, we decompose the outputs contributing to \smash{$L^\infty_{i_j}$} into two parts: elements of $L^j$, which dominate, and hallucinations outside $L^j$. The first contribution is small by (ii), and the second can be made small with high probability via Markov's inequality applied at an appropriately chosen schedule of times $t_j$ from (i), yielding a failure probability that decays exponentially in $j$.

    A Borel--Cantelli argument then implies that, almost surely, the hallucination contribution is small for all sufficiently large $j$. It follows that along the subsequence \smash{$\{i_j\}_{j\in\bbN}$}, the fraction of timely outputs in \smash{$L^\infty_{i_j}$} can be made arbitrarily small, yielding $\mu^{\el}_{\low}(\calA,D)=0$.
\end{proof}

\section{Timely Dense Language Generation with Sparse Hallucination}
\label{sec:positive_result}

In this section, we present a (randomized) timely generative algorithm for dense language generation with sparse hallucination for any super-linear deadline function. 
We prove that our algorithm achieves optimal density guarantees and has an essentially optimal hallucination rate. To simplify upcoming statements, we define the notion of a feasible profile motivated by \cref{lem:deadline-hallucination-impossibility}.
\begin{definition}[Feasible Profile]
    We say $(D,H)$ is a feasible profile if $D:\bbN\to\bbN$ is a monotone non-decreasing, discrete convex,\footnote{Discrete convexity of $D$ is not necessary for the argument, but we assume it here for clarity of presentation.} super-linear function, and $H:\bbN\to\bbR_{\geq 0}$ is a monotone non-increasing vanishing function satisfying $H(t)=\omega(D^{-1}(t)/t)$. 
\end{definition}
We now turn to the positive result, which shows that the impossibility phenomena identified in the previous section can be overcome when the deadline function grows sufficiently quickly.

\begin{theorem}[Dense Generation in Nearly-Linear Time]\label{thm:lower-density-guarantee}
    For any feasible profile $(D,H)$,  there exists a randomized generation algorithm $\calA$ that achieves density $\mu^{\el}_{\low}(\calA,D)=1/2$ and satisfies
    \[
        \sup_{K\in \calL}\sup_{E\in\calE(K)} H_{\calA}^{\calL,K,E}(t) \leq H(t)
    \]
    for all countable collections $\calL$ and sufficiently large $t$.
\end{theorem}

This theorem characterizes a sharp tradeoff between the growth of the deadline function and the minimial hallucination rate sufficient to achieve optimal density. In particular, the rate of hallucinations by time $t$ can be set to lie in the range 
\[
    D^{-1}(t)/t
    \;\;\ll\;\; H(t) \;\;\ll\;\; 1.
\]
The upper bound $H(t)=o(1)$ captures our requirement of a vanishing hallucination rate. The lower bound of $H(t)=\omega(D^{-1}(t)/t)$ captures the constraint imposed by our deadline function and demonstrates that our construction is essentially optimal as established in \cref{lem:deadline-hallucination-impossibility}. Crucially, as the deadline function grows more quickly, this constraint weakens, thereby allowing the generator to hallucinate at a much slower rate while still achieving optimal lower density. 

To simplify our proof, we introduce a weaker notion of timely density that relaxes the strict element-wise requirement to one more amenable to analysis. 
\begin{definition}[Prefix-Wise Density]\label{def:prefix-metric}
    For any two sequences $S$ and $T$ and function $F : \bbN\to \bbN$, define
    \begin{align*}
        \mu^\pfx_i(S, R, F) := \frac{|S_{F(i)} \cap R_i|}{i}.
    \end{align*}
    Then for a generator $\calA$ and deadline function $D:\bbN \to \bbN$ define
    \begin{align*}
        \mu^\pfx_{\low}(\calA, D)
        := \inf\left\{\inf_{K \in \calL}\inf_{E \in \calE(K)} \liminf_{i \to \infty} \mu^\pfx_i(\calA(E), K, D) : \calL \text{ is countable}\right\}.
    \end{align*}
\end{definition}

Under the original metric $\mu^{\el}_i$, the generator must generate each element $k_i$ of the target sequence by its designated deadline $t=D(i)$ in order to receive credit, enforcing a per-element notion of timeliness. The prefix density metric measures how many elements of the prefix $K_i$ have been generated by time $D(i)$, regardless of when those elements were produced. Consequently, an element $k_j$ for $j<i$ that is generated after its own deadline $D(j)$ may still contribute to the generator’s performance when measured against a larger prefix $K_i$ provided its generation time falls before $D(i)$. 

The next lemma shows that working with this prefix-wise metric is without loss: any constant lower bound on prefix-wise density for all super-linear deadline functions suffices to obtain the same guarantee under the original element-wise notion.

\begin{lemma}[Reduction Between Density Metrics]\label{lem:reduction}
    Fix any density $\rho\geq 0$ and feasible profile $(D_{\el},H_{\el})$. There exists a feasible profile $(D_\pfx,H_\pfx)$ such that, for every countable $\calL$, any generator $\calA$ satisfying $\mu^\pfx_{\low}(\calA,D_\pfx)\geq \rho$ and
    \[
        \sup_{K\in\calL}\sup_{E\in\calE(K)} H_{\calA}^{\calL,K,E}(t)\leq H_\pfx(t)
    \]
    for all sufficiently large $t$, also satisfies $\mu^{\el}_{\low}(\calA,D_{\el})\geq \rho$ and
    \[
        \sup_{K\in\calL}\sup_{E\in\calE(K)} H_{\calA}^{\calL,K,E}(t)\leq H_{\el}(t)
    \]
    for all sufficiently large $t$.
\end{lemma}

This reduction also implies that any impossibility result for the $\mu^{\el}_{\low}$ carries over to the $\mu^{\pfx}_{\low}$.

It therefore remains to construct an algorithm that achieves density $\rho$ with respect to the prefix metric for any feasible hallucination rate. To do so, we further reduce the problem to the setting without deadlines. In particular, for any super-linear deadline function, we present a black-box construction that converts any eventually consistent dense generator that performs well in the limit into a timely generator with vanishing hallucination rate, while preserving its prefix-wise density guarantees.

\begin{lemma}[Black-Box Reduction]\label{lem:algorithm-guarantee}
    Fix any density $\rho\geq 0$ and feasible profile $(D,H)$. For any deterministic eventually consistent generator $\calB$ such that $\mu_{\low}(\calB,\infty)=\rho$, there exists a randomized algorithm $\calA=\calA_\calB$ achieving \smash{$\mu^{\pfx}_{\low}(\calA,D)\geq\rho$} almost surely and for all countable collections $\calL$
    \[
        \sup_{K\in\calL}\sup_{E\in\mathcal{E}(K)}H_{\calA}^{\calL,K,X}(t)\le H(t)
    \]
    for all sufficiently large $t$. Furthermore, $\calA$ only makes black-box calls to $\calB$.
\end{lemma}

Combining these two lemmas with the eventually consistent generator of Kleinberg and Wei \cite{KW}, which achieves lower density $1/2$ in the limit, yields the desired result. We defer a deeper discussion of \cref{lem:reduction} and \cref{lem:algorithm-guarantee} to the next two subsections.

\subsection{Discussion of \cref{lem:reduction}: From Element-Wise to Prefix-Wise Density}\label{sec:reduction}

At a high level, the reduction can be viewed as a controlled time-dilation argument. Rather than attempting to satisfy the element-wise deadlines directly, we construct a slower-growing prefix time scale under which meeting prefix deadlines is sufficient to meet almost all element-wise deadlines.

Given a feasible element-wise profile $(D_{\el},H_{\el})$, we construct a prefix deadline function $D_\pfx$ that grows more gradually. Then, for each prefix index $i$, we compare the prefix deadline $D_\pfx(i)$ with the element-wise deadlines $D_{\el}(j)$ for $j\leq i$. This comparison identifies how many elements have deadlines that occur before the prefix deadline. By construction, any element whose deadline occurs after $D_\pfx(i)$ automatically satisfies its constraint if generated by time $D_\pfx(i)$. Thus, among elements of $K_i$ generated by this time, only those with earlier deadlines may fail to meet their deadline. By choosing $D_\pfx$ to grow sufficiently slowly, we ensure that this set is a vanishing fraction of the prefix, so the loss in density is negligible.

The more delicate issue is controlling hallucinations under this change of time scale. Since the algorithm only provides guarantees at prefix deadlines, to control hallucinations at time $t$ we map to the smallest index $i$ with $D_\pfx(i)\ge t$ and apply the bound at time $D_\pfx(i)$. Therefore, the relevant bound at time $t$ is given by $H_\pfx(D_\pfx(i))$, or equivalently \smash{$H_\pfx(D_\pfx(D_\pfx^{-1}(t)))$}. The challenge is that this replaces control at time $t$ with control at a later time $D_\pfx(i)$, and so the guarantee depends on how much larger $D_\pfx(i)$ is than $t$. 

The choice of $D_\pfx$ must therefore balance two competing requirements: it must remain small relative to $D_{\el}$ so that only a vanishing fraction of elements have earlier deadlines, while not growing so quickly that evaluating the hallucination bound at $D_\pfx(i)$ significantly inflates the budget. This necessitates a coupled choice of $D_\pfx$ and the prefix-wise hallucination rate $H_\pfx$.

Below we give a proof sketch of \cref{lem:reduction}. A full proof can be found in \cref{app:reduction}.

\begin{proof}[Proof Sketch]
    We make the above argument precise. For each prefix index $i$, define
    \[
        r(i):=\max\{j\leq i : D_{\el}(j) < D_\pfx(i)\},
    \]
    to be the number of elements whose element-wise deadlines occur before the prefix deadline $D_\pfx(i)$.

    By definition, any element among $\{k_{r(i)+1},\dots,k_i\}$ that is generated by time $D_\pfx(i)$ automatically satisfies its element-wise deadline. Thus, among the elements of $K_i$ generated by time $D_\pfx(i)$, only those in $K_{r(i)}$ may fail to meet their deadlines. By choosing $D_\pfx$ to grow sufficiently slowly relative to $D_{\el}$, we ensure that $r(i)=o(i)$, so this loss is negligible when measuring density.

    To control hallucinations, we relate time $t$ to the prefix index via
    $
        \tau(t):=\min\{n\in\bbN : D_\pfx(n)> t\}.
    $
    By definition of $\tau(t)$, we have $t < D_\pfx(\tau(t))$. Thus, any hallucination incurred by time $t$ is also incurred by time $D_\pfx(\tau(t))$.

    We construct a prefix hallucination rate $H_\pfx$ that is feasible and satisfies $H_\pfx(t)\leq H_{\el}(t)$ for all sufficiently large $t$. The key step is a delicate choice of $D_\pfx$ so that $\tau(t)$ grows sufficiently slowly relative to $t$, allowing one to construct a feasible prefix hallucination rate $H_\pfx$ such that $H_\pfx(D_\pfx(\tau(t))) \leq H_{\el}(t)$.

    Combining these observations, we obtain that only a vanishing fraction of elements fail to meet their deadlines, and that hallucination bounds transfer, establishing the desired reduction.
\end{proof}

\subsection{Discussion of \cref{lem:algorithm-guarantee}: A Randomized Algorithm for Prefix-Wise Density}

At a high level, the procedure in \cref{alg:speculative-blackbox-generation-general} can be viewed as a controlled mixture of two processes. Rather than relying solely on the blackbox generator $\calB$, which achieves density $\rho$ only in the limit, we introduce a sparse sequence of speculative steps that directly target elements of the prefix. The goal is to ensure that, by time $t=D(i)$, sufficiently many elements of $K_i$ have been generated, while keeping the total number of speculative steps small enough to stay within the hallucination constraint.

\begin{algorithm}
\caption{\textsf{S}peculative \textsf{B}lackbox \textsf{G}eneration}
\label{alg:speculative-blackbox-generation-general}
\begin{algorithmic}[1]
\State \textbf{Parameters:} Deadline function $D:\bbN\to\bbN$, hallucination rate $H:\bbN\to\bbR_{\geq 0}$, blackbox $\calB$
\State Choose an unbounded monotone non-decreasing $C:\bbN\to\bbN$ such that $\sum_{m\le m(t)} C(m) = o(t H(t))$ and $C(m) = o(D(m)-D(m-1))$.
\For{$m=1,2,\dots$}
    \State $\delta_m\gets \min\left\{1,\frac{10\,C(m)}{D(m)-D(m-1)}\right\}$
    \For{$t=D(m-1)+1,\dots,D(m)$}
        \State Sample $\decision_t\sim\Bern(\delta_m)$ and $\choice_t\sim\Uniform([C(m)])$
        \If{$\decision_t=0$}
            \State $o_t\gets \calB(e_{1:t})$ \Comment{$t$-th output of black box}
        \Else
            \State $o_t \gets \mathsf{NextUnused}(L^{\choice_t})$ \Comment{first element in $L^{\choice_t}$ not yet observed/outputted}
        \EndIf
        \State Generate $o_t$
    \EndFor
\EndFor
\end{algorithmic}
\end{algorithm}

We partition time into epochs $[D(m-1)+1,D(m)]$. Within each epoch $m$, we maintain a list $L^1,\ldots,L^{C(m)}$ of candidate languages, and during each round we speculate with probability
\[
    \delta_m \asymp \frac{C(m)}{D(m)-D(m-1)}.
\]
On a speculative step at time $t$, the algorithm selects an index $\choice_t\in[C(m)]$ uniformly at random and outputs the \emph{next unused} element of $L^{\choice_t}$. By this choice of $\delta_m$, the algorithm performs $\Theta(C(m))$ speculative steps in epoch $m$. Moreover, since $C(m)=o(D(m)-D(m-1))$, we have $\delta_m\to 0$, so speculation—and therefore hallucination—becomes increasingly rare across epochs.

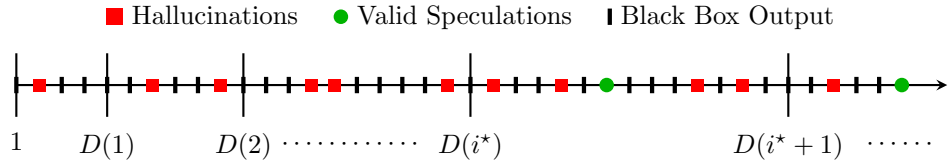
\begin{figure}[ht]\label{fig:sbg-intuition}
    \centering

    \begin{tikzpicture}[x=1.2cm,y=1.2cm,>=stealth]

  \def\Dtick{0.35}    
  \def\Atick{0.09}    
  \def\RedHalf{0.07}  
  \def\GreenRad{0.08} 
  \def\Dref{0.1 - \Dtick} 

  \draw[thick,->,black] (0,0) -- (10.25,0);

  \foreach \x/\lab in {
    0/{1},
    1/{D(1)},
    2.5/{D(2)},
    5/{D(i^\star)},
    8.5/{D(i^\star+1)}
  }{
    \draw[thick, black] (\x,\Dtick) -- (\x,-\Dtick);
    \node[below=6pt] at (\x,\Dref) {$\lab$};
  }

  \node at (3.7,-0.65) {$\cdots \cdots \cdots \cdots$};
  \node at (9.75,-0.65) {$\cdots \cdots$};


  \node at (5,0.75) {%
    \tikz[baseline=-0.6ex]{\fill[red] (-0.09,-0.09) rectangle (0.09,0.09);}~Hallucinations
    \quad
    \tikz[baseline=-0.6ex]{\fill[green!70!black] (0,0) circle (\GreenRad);}~Valid Speculations
    \quad
    \tikz[baseline=-0.6ex]{\draw[ultra thick] (0,-0.09) -- (0,0.09);}~Black Box Output
  };

  \foreach \x in {
    0.5,0.75,1.0,1.25,1.75,2.0,2.5,2.75,3.0,3.75,
    4.0,4.25,4.5,5.0,5.5,5.75,6.25,6.75,7.0,7.25,
    7.75,8.25,8.5,8.75,9.25,9.5
  }{
    \draw[ultra thick] (\x,\Atick) -- (\x,-\Atick);
  }
  \draw[ultra thick] (0,\Atick) -- (0,-\Atick);

  \foreach \x in {0.25,1.5,2.25,3.25,3.5,4.75,5.25,6,7.5,8,9}{
    \fill[red] (\x-\RedHalf,-\RedHalf) rectangle (\x+\RedHalf,\RedHalf);
  }

  \foreach \x in {6.5,9.75}{
    \fill[green!70!black] (\x,0) circle (\GreenRad);
  }

\end{tikzpicture}
    
    \caption{Visual for behavior of \textsf{SBG} algorithm: every tick mark represents a generated element and larger tick marks represent deadlines for corresponding prefixes of the target language.}
    \label{fig: Intuition for SBG}
\end{figure}

Since $C(m)\to\infty$, the true target $K=L^{i^\star}$ is eventually included among these candidates. From that point onward, a $1/C(m)$ fraction of speculative steps in epoch $m$ target $K$, so each epoch yields a constant expected number of fresh elements of $K$. Consequently, the total number of successful speculations by time $D(i)$ is $\Omega(i)$. This behavior is illustrated in \cref{fig: Intuition for SBG}.

At the same time, we make a strategic choice of $C$ so that $\sum_{k\le m} C(k)=o(tH(t))$. This ensures that by time $D(m)$ the cumulative number of hallucinations remains within the budget $tH(t)$. Thus, speculation is frequent enough to ensure linear coverage of $K$, while remaining negligible relative to the hallucination budget. 

Below we give a proof sketch of \cref{lem:algorithm-guarantee}. A full proof can be found in \cref{app:algorithm-guarantee}.

\begin{proof}[Proof Sketch]
    Fix any $K\in\calL$ and any enumeration $E\in\calE(K)$. We first bound the expected number of hallucinations of the generator. By our assumption that the blackbox is an eventually consistent generator, it can contribute at most finitely many outputs outside $K$. Furthermore, the expected number of speculative outputs by time $t$ is
    \[
        \sum_{m\le m(t)} \delta_m\cdot(D(m)-D(m-1))
        \asymp
        \sum_{m\le m(t)} C(m),
    \]
    which is $o(tH(t))$ by the choice of $C$. Thus $H_{\calA}^{\calL, K, E}(t)\le H(t)$ for all sufficiently large $t$.

    We next argue that the algorithm speculates sufficiently many elements of $K$. Once $C(m)\ge i^\star$, each round in epoch $m$ speculates a fresh element of $K$ with probability
    \[
    \frac{\delta_m}{C(m)} \asymp \frac{1}{D(m)-D(m-1)}.
    \]
    Summing over an epoch yields a constant expected number of such events, and hence by a Chernoff bound, for sufficiently large $i$, by time $D(i)$ the algorithm speculates $i$ distinct elements of $K$ in expectation with exponentially high probability. It then follows that the probability of the complementary event is exponentially low, which allows us to apply a Borel-Cantelli argument to show that after sufficiently large $i$, the algorithm speculates $i$ elements from $K$ by time $D(i)$ almost surely.
    
    It remains to show that the contribution of $\calB$ is not destroyed by the random deletions induced by speculation. Denote time $t_i$ to be the timestep by which the output of the blackbox and the output of the adversary would have contained all of $K_i$. By the blackbox density guarantee we know that for sufficiently large $i$, the blackbox produces $(\rho-o(1))i$ elements of $K_i$ by the time $t_i$.
    
    We distinguish two cases. If $t_i\ge D(i)$, the guarantee follows directly from the aforementioned guarantee that we speculate sufficiently many elements of $K$ by time $D(i)$ and from the fact that the adversary could not have claimed too many elements in $K_i$ by this time. Otherwise, we partition the blackbox outputs into early and late rounds where a round is late if $\delta_{m(t)}$ is sufficiently small. The early rounds contribute only $o(i)$ elements. On the late rounds, the retention probability is $1-\delta_m=1-o(1)$, and a martingale concentration argument (via Freedman's inequality) shows that–because the decision for retention is independent of the blackbox's output being from $K_i$–only an $o(i)$ fraction of these outputs are discarded with exponentially small failure probability.
    
    Combining these two cases yields $|\calA(X)_{D(i)}\cap K_i|\ge (\rho-o(1))i$ with exponentially high probability and another Borel--Cantelli argument on the complementary event completes the proof.
\end{proof}

In fact, the algorithm satisfies a stronger guarantee: the union of the outputs of the generator and the adversary achieves lower density $1$ under the deadline function $D$ almost surely. Indeed, in the proof of the lower density guarantee for \cref{alg:speculative-blackbox-generation-general}, we show that for all sufficiently large $i$, by time $D(i)$ the generator has almost surely produced at least $i$ elements of $K$. Since these are the earliest unused elements of $K$, each such output either lies in $K_i$ or corresponds to an element of $K_i$ that must have been produced earlier by the generator or the adversary. In either case, every element of $K_i$ is accounted for by time $D(i)$, yielding density $1$.

Finally, the growth of $C$ determines when the generator can be guaranteed to have good density. The generator can only speculate from the target $K=L^{i^\star}$ once it is included in the candidate list, which occurs at epoch $\min\{m : C(m)\ge i^\star\}$. When $C$ grows slowly, this inclusion is delayed, yielding an initial phase in which the generator relies solely on $\calB$, followed by the phase in which speculation produces a linear number of elements of $K$. Therefore, larger delay windows allow $C$ to grow more quickly, so the generator can be guaranteed to achieve good density earlier.

\section{Sporadically Dense Generation with Consistent Generators}\label{sec:upper-density}

While the impossibility results of \cref{sec:impossibility} establish that there exist simple language families $\calL_{chain}$ for which no consistent generator can achieve density across all finite prefixes of every target language $K \in \calL_{chain}$, this limitation can be circumvented under a weaker notion that requires density only infinitely often at isolated indices. We formally define this notion, first introduced by \cite{KW}, below.
\begin{definition}[Upper Density]
    The upper density of a generation algorithm $\calA$ is defined as
    \begin{align*}
        \mu_{\up}^{\el}(\calA, D) &= \inf\left\{\inf_{K\in\calL}\inf_{E\in\calE(K)} \limsup_{i\to\infty} \mu^{\el}_{i}(\calA(E),K,D): \calL \text{ is a countable language collection}\right\}.
    \end{align*}
\end{definition}
Under this notion, it suffices for the generator to achieve high density along an infinite subsequence of indices, rather than maintaining density uniformly over all sufficiently large prefixes. Interestingly this relaxation suffices for eventually-consistent generators to obtain the optimal density in a timely fashion. 

\begin{theorem}[Separation of Upper and Lower Density]\label{thm:total enumeration consistent timely generation}
    Let $D(i) = i$ be the identity function. There exists an eventually consistent generator $\calA$ such that $\mu_{\up}^{\el}(\calA, D) = 1/2$.
\end{theorem}

This result is strong in that the smallest deadline one can hope to obtain optimal density within is the identity function. This result also implies the first quantitative separation between the upper density and lower density metrics. We defer the proof to \cref{app:upper-density}, where we show that an adaptation of the algorithm of \cite{KW} achieving upper density in the limit can do so under the tightest deadline.

\section{Discussion and Open problems}

Our results highlight a central and somewhat counterintuitive phenomenon: hallucination is not merely an artifact of imperfect generation, but a necessary ingredient for achieving breadth in a manner that reflects inductive bias. In particular, our impossibility results show that any eventually consistent generator must sacrifice timely density, while our constructions demonstrate that allowing controlled hallucination restores the ability to generate densely in a timely fashion. This suggests that hallucination plays a fundamental role in enabling exploration. A deeper investigation of this phenomenon remains an important direction for future work.

For clarity of exposition, we have presented our results in the setting where the adversary sequentially reveals a total enumeration of the target language. However, the scope of our techniques extends well beyond this setting: the reduction is agnostic to the observed sequence and applies whenever a blackbox generator can achieve nontrivial lower density in the limit on any admissible sequence of observations that the adversary may provide (see \cref{app:positive-results} for more details). As an immediate consequence, our results extend to the partial enumeration model of \cite{KW2}, and contamination models of \cite{mehrotra2025languagegenerationinfinitecontamination} yielding generators that achieve nontrivial density in a timely manner in these settings as well.

Finally, several quantitative questions remain open. While we characterize the optimal achievable lower density for arbitrary super-linear deadline functions, our notion of timeliness is tied only to the generation time of individual strings. One could additionally consider timeliness with respect to a preference ordering over the language hypotheses themselves, asking how the time required for meaningful generation depends on the position of the target language within the ordering. Simultaneously controlling timeliness with respect to both individual strings and the underlying language family would yield a more refined dual objective for timely language generation. This leads us to the following conjecture.

\begin{conjecture}[The Golden-Ratio Barrier for Timely Language Generation]
    In our construction, if the deadline function is $D(i)=i^{1+\alpha}$, then the generator does not begin speculating from the target language $K$ until roughly time $s^{1/\alpha}$, assuming $K=L^s$. Balancing these two exponents leads to the relation $\alpha^2+\alpha=1$, whose positive solution is $\alpha=(\sqrt5-1)/2$. Equivalently, the corresponding delay exponent is the golden ratio $\varphi=(1+\sqrt5)/2$. This raises the intriguing possibility of a ``golden ratio law'' for timely language generation: namely, that the optimal tradeoff between string-level timeliness and language-level timeliness is fundamentally governed by the golden ratio.
\end{conjecture}

\bibliographystyle{abbrv}
\bibliography{bibliography}

\appendix
\crefalias{section}{appendix}
\crefalias{subsection}{appendix}
\crefname{appendix}{Appendix}{Appendices}
\Crefname{appendix}{Appendix}{Appendices}
\section{Additional Discussion of Related Work}
\label{app:related}

The  literature on language generation is vast and multifaceted. We therefore cannot do all of it justice, and emphasize only the most relevant papers to our work. 

We work in the framework  of language generation in the limit, as posed by introduced by Kleinberg and Mullainathan~\cite{KM}. This builds on analogous questions for language identification introduced by Gold~\cite{Gold} and studied by Angluin~\cite{Angluin}. 

In contrast to language identification, \cite{KM} show that consistent  generation is possible in general. However, somewhat tellingly their algorithm suffers mode collapse in a manner which seems attributable to a tension with consistency; this sets the stage for much of the subsequent work exploring the interplay  between consistency and breadth---a theme which permeates our paper as well.

Several notions of breadth have been considered in the theoretical literature on language generation.  Strong notions which require the generator to eventually generate from all or almost all unseen strings in $K$ were studied by Kalavasis et al~\cite{kalavasis1,kalavasis2}, leading to largely negative results in general settings for consistent generators. Similar impossibility results also hold for a notion of breadth, called exhaustive generation, as shown by Peale et al~\cite{peale}. 

Things took a turn for the positive with the work of Kleinberg and Wei~\cite{KW}, which proposed measuring breadth with respect to a fixed global order on strings---this is the ``preference order'' we also consider in our paper. They showed that nontrivial breadth guarantees---in the sense of positive density---were achievable by consistent generators with respect to such an order. Subsequent   work by the same authors in \cite{KW2} strengthens those results to an optimal density of $1/2$.

It is this latter line of work which is most related to our approach in this paper. We start with the breadth notions in \cite{KW,KW2} and naturally strengthen them to incorporate timely generation, measured with respect to the same preference order on strings. We find that this modification changes the picture, implying strong impossibility results for consistent generation with timely breadth that are reminiscent of the negative results in aforementioned previous approaches to reconciling breadth and consistency. However, the presence of the global order on strings proves particularly useful here: by prioritizing preferred strings carefully over time, we are able to obtain simultaneous and meaningful guarantees both on timely breadth as well as the hallucination rate.

There are many other works which consider largely orthogonal concerns to ours in language generation. This includes work which explores stronger guarantees on the time to consistency~\cite{cb1, cb2, hanneke, Li}, as well as works which strengthen the adversary to allow for partial or noisy training data~\cite{bai,KW2, mehrotra2025languagegenerationinfinitecontamination, ramanAndRaman}. The latter group of works is consequential to ours to the extent that our positive result---by virtue of its black-box nature---is well-positioned to port over our timely breadth guarantees to such more general settings as we will show in \cref{app:corollaries}.

Finally, we would be remiss not to mention some of the work examining statistical perspectives on hallucination. We single out ~\cite{kalai, omar} as arguing for the inevitability of hallucination in both theory and practice. This is in-keeping with a key message of our paper, that hallucination serves to enable learning with breadth and can be minimized yet not completely eliminated.

\section{Examples of Measure-Zero Chains}\label{app:examples}

In this section we provide examples of language families that fall under \cref{def:nested family}. First we restate the definition of a Measure-Zero Chain.

\begin{definition}[Restatement of \cref{def:nested family}]
    A collection $\calL_{chain} = \{L^\infty\} \cup \bigcup_{j\geq 1} \{L^j\}$ forms a zero support chain if $L^1 \subseteq L^2 \subseteq \cdots \subseteq L^\infty$ is a nested sequence of languages for which:
    \[ 
        \bigcup\nolimits_{j\geq 1} L^j = L^\infty \qquad \text{and} \qquad \liminf_{i\to \infty} \mu_i(L^j, L^\infty, \infty) = 0 \qquad \forall j \in \bbN.
    \]
\end{definition}

We now provide a very simple language family which satisfies this definition.

\begin{example}[Block Partition Construction]\label{ex:block-partition}
    Let $U=\bbN$ with the canonical order, and partition $U$ into consecutive disjoint blocks
    \[
        U=\bigsqcup_{m\geq 0} B_m, \qquad B_m = [2^{m}, 2^{m+1}-1],
    \]
    so that $|B_m|=2^m$. For each $m\geq 0$, let $b_m^\star \in B_m$ denote the first element of $B_m$ in the canonical order, i.e., $b_m^\star = 2^m$.

    Define the limiting language $L^\infty := U$, and for each $j\in\bbN$,
    \[
        L^j := \left(\bigcup_{m\le j} B_m\right) \;\cup\; \left(\bigcup_{m>j} \{b_m^\star\}\right).
    \]

    Then, by construction $L^1 \subseteq L^2 \subseteq \cdots \subseteq L^\infty$ and $\bigcup_{j\ge 1} L^j = L^\infty$. Moreover, for each fixed $j$,
    \[
        \liminf_{i\to \infty} \mu_i(L^j, L^\infty, \infty) = 0,
    \]
    since, up to index $n$, the number of blocks is on the order of $\log n$, and beyond the first $j$ blocks each contributes at most one element to $L^j$. Thus, $|L^j \cap [n]| \le O(\log n)$, while $|L^\infty \cap [n]| = n$, implying that the density vanishes.

    This construction is not tied to the particular dyadic choice of block sizes. More generally, the same example goes through whenever the block sizes $|B_n|$ grow according to any super-constant function of $n$, and the distinguished element $b_n^\star$ may be chosen arbitrarily within each block. The only property used is that each fixed $L^j$ contains all elements from the first finitely many blocks, but only one element from each later block, so its density inside $L^\infty$ vanishes as the block sizes grow.
\end{example}

Interestingly, another instructive example of a measure-zero chain arises in the construction of Kleinberg and Wei in their initial study of lower density in the limit. Below we provide this example and show that it satisfies the definition of a measure-zero chain.

\begin{example}[Marker-Interval Construction (Example 7 in \cite{KW})]\label{ex:marker-interval}
    Let $U=\bbN$ with the canonical order. Define a sequence of marker locations $\{a_i\}_{i\ge 0}$ by $a_0 = 0$, and $a_i = 3^i$ for $i\ge 1$. For each $j\in\bbN$, define the language
    \[
        L^j := \bigcup_{i\ge 0} [a_i, a_i + j],
    \]
    where we can think of this construction as attaching an interval of length $j$ at each marker $a_i$. Let $L^\infty := \bbN$. 
    
    Then, by construction $L^1 \subseteq L^2 \subseteq \cdots \subseteq L^\infty$ and $\bigcup_{j\ge 1} L^j = L^\infty$. Moreover, for each fixed $j$,
    \[
        \liminf_{i\to\infty} \mu_i(L^j, L^\infty, \infty) = 0,
    \]
    since, up to index $n$, the number of markers is on the order of $\log n$, and each contributes at most $j+1$ elements to $L^j$. Thus, $|L^j \cap [n]| \le (j+1)\cdot O(\log n)$, while $|L^\infty \cap [n]| = n$, implying that the density vanishes.

    This construction is not tied to the specific choice $a_i=3^i$. More generally, like in \cref{ex:block-partition}, the same example goes through for any sequence of markers $\{a_i\}$ that grows super-constantly, with intervals of length $j$ attached at each marker. The only property used is that, up to index $n$, the number of markers is $o(n)$, so each fixed $L^j$ captures only $o(n)$ elements of $L^\infty=\bbN$, implying that the density vanishes.
\end{example}

More generally, both constructions above are instances of a common pattern: we partition $U$ into disjoint pieces whose sizes grow, and define $L^j$ by retaining each piece entirely up to some index $j$, while selecting only a vanishing fraction of each subsequent piece. Whenever the total contribution of these partial selections is negligible compared to the full pieces, the resulting chain satisfies
\[
    \liminf_{i\to\infty} \mu_i(L^j, L^\infty, \infty) = 0.
\]
Note that such a language family need only be contained in $\calL$ for the impossibility result of \cref{lem:deadline-hallucination-impossibility} to kick in. 

\section{Impossibility Results}\label{app:impossibility}

Recall the definition in \cref{sec:impossibility} of a measure-zero chain defined in measure-zero chain, which we restate for convenience in \cref{app:examples}. We will now show that any language collection $\calL$ which contains a measure-zero chain witnesses the impossibility result of \cref{lem:deadline-hallucination-impossibility}

\begin{lemma}[Restatement of \cref{lem:deadline-hallucination-impossibility}]
    For any monotonically increasing deadline function $D:\bbN\to\bbN$,  there exists a language family $\calL$ such that for any randomized generator $\calA$, if 
    \[
        \sup_{K\in\calL}\sup_{E\in\calE(K)}H_{\calA}^{\calL,K,E}(t)
        =
        o\!\left(\frac{D^{-1}(t)}{t}\right),
    \]
    then $\mu^{\el}_{\low}(\calA,D)=0$ almost surely.
\end{lemma}

\begin{proof}
    We will actually prove this result by directly proving that $\mu^{\pfx}_{\low}(\calA,D)=0$ where $\mu^{\pfx}_{\low}$ is formally defined in \cref{def:prefix-metric}.

    Fix a chain $\calL_{chain}=\{L^\infty\}\cup\{L^j:j\in\bbN\}\subseteq\calL$ as in \cref{def:nested family}. We construct a total enumeration $X^\infty$ of $L^\infty$ incrementally, in stages indexed by $j\in\bbN$. At stage $j$, we temporarily pretend that the target language is $L^j$: we choose an enumeration of $L^j$ extending the prefix constructed so far, follow it for a sufficiently long finite interval ending at time $t_j$, and then pass to stage $j+1$. The times $t_j$ will be chosen so that the generator produces only a negligible fraction of elements from \smash{$L^\infty_{i_j}$} by time $t_j$, where $i_j := D^{-1}(t_j)$. This will force \smash{$\liminf_{i\to\infty}\mu^{\pfx}_i(\calA(X^\infty),L^\infty,D)=0$} almost surely.

    We represent the randomized algorithm as a deterministic map indexed by a random seed: for each seed $R$, let $\calA_R$ denote the corresponding deterministic algorithm.

    Let $t_0:=0$. Suppose inductively that stages $1,\dots,j-1$ have already been defined, so that the current prefix of $X^\infty$ has length $t_{j-1}$. Since the chain is nested and $L^{j-1}\subseteq L^j$, this prefix can be extended to a total enumeration $X^{(j)}$ of $L^j$.

    By the assumed hallucination rate guarantee on $X^{(j)}$ with respect to $L^j$, we have
    \[
        H_{\calA}^{\calL,L^j,X^{(j)}}(t)
        =
        o\!\left(\frac{D^{-1}(t)}{t}\right).
    \]
    In particular,
    \[
        \bbE\!\left[
            \frac{|\calA_R(X^{(j)})_{t}\setminus L^j|}{D^{-1}(t)}
        \right]
        =
        \frac{t}{D^{-1}(t)} \cdot H_{\calA}^{\calL,L^j,X^{(j)}}(t)
        =
        o(1).
    \]

    Furthermore, by \cref{def:nested family}, $\liminf_{i\to\infty}\mu_i(L^j,L^\infty,\infty)=0$. Hence we may choose a sufficiently large time $t_j \ge t_{j-1}$ and define $i_j := D^{-1}(t_j)$ such that
    \[
        \frac{|L^j\cap L^\infty_{i_j}|}{i_j} \le 2^{-(j+1)}
        \qquad\text{and}\qquad
        \bbE\!\left[
            \frac{|\calA_R(X^{(j)})_{t_j}\setminus L^j|}{i_j}
        \right]
        \le 2^{-(2j+1)}.
    \]

    We additionally choose $X^{(j)}$ so that every element of $L^j$ that lies among the first $i_j$ elements of $L^\infty$ and has not yet appeared in the prefix \smash{$X^\infty_{t_{j-1}}$} is enumerated by time $t_j$. Increasing $t_j$ if necessary preserves the above inequalities. We then let $X^\infty$ agree with $X^{(j)}$ on the interval $\{t_{j-1}+1,\dots,t_j\}$. This recursively defines an infinite sequence $X^\infty$. Since $\bigcup_{j\ge1}L^j=L^\infty$, every element of $L^\infty$ appears at some finite stage, and hence $X^\infty$ is a total enumeration of $L^\infty$.

    Now fix $j$. Since $X^\infty$ and $X^{(j)}$ agree through time $t_j$, for every fixed seed $R$, the generator produces identical outputs on both runs up to time $t_j$. Therefore, $\calA_R(X^\infty)_{t_j} = \calA_R(X^{(j)})_{t_j}$. Hence,
    \begin{align*}
        \mu^{\pfx}_{i_j}(\calA_R(X^\infty),L^\infty,D)
        &= \frac{|\calA_R(X^\infty)_{D(i_j)}\cap L^\infty_{i_j}|}{i_j}\\
        &\le \frac{|\calA_R(X^\infty)_{t_j}\cap L^\infty_{i_j}|}{i_j}\\
        &\le \frac{|L^j\cap L^\infty_{i_j}|}{i_j}
        + \frac{|\calA_R(X^{(j)})_{t_j}\setminus L^j|}{i_j}\\
        &= 2^{-(j+1)} + \frac{|\calA_R(X^{(j)})_{t_j}\setminus L^j|}{i_j},
    \end{align*}
    where it suffices to show that the second term is eventually small almost surely. 
    To bound the second term, we consider the event $E_j$ that the generator outputs more than a $2^{-(j+1)}$ fraction of elements outside $L^j$ when run on $X^{(j)}$ by time $t_j$. Then, by Markov's inequality, we have
    \[
        \Pr(E_j) = \Pr\left(\frac{|\calA_R(X^{(j)})_{t_j}\setminus L^j|}{i_j}>2^{-(j+1)}\right)\leq  \frac{2^{-(2j+1)}}{2^{-(j+1)}}=2^{-j}.
    \]

    Since this probability is exponentially decreasing in $j$, the first Borel--Cantelli lemma implies that only finitely many of the events $E_j$ occur. Hence, almost surely, for all sufficiently large $j$,
    \[
        \frac{|\calA_R(X^{(j)})_{t_j}\setminus L^j|}{i_j}
        \le 2^{-(j+1)}.
    \]

    Therefore, almost surely, for all sufficiently large $j$,
    \[
        \mu^{\pfx}_{i_j}(\calA_R(X^\infty),L^\infty,D)
        \le 2^{-(j+1)} + 2^{-(j+1)} = 2^{-j}.
    \]

    Since this bound then holds along the subsequence $\{i_j\}_{j\in\bbN}$, we can therefore conclude that 
    \[
        \liminf_{i\to\infty}\mu^{\pfx}_i(\calA_R(X^\infty),L^\infty,D)=0
    \]
    almost surely. 
    Then, because $L^\infty\in\calL$ and $X^\infty$ is a total enumeration of $L^\infty$, this implies $\mu^{\pfx}_{\low}(\calA_R,D)=0$ almost surely. Finally, since $\mu^{\el}_i \le \mu^{\pfx}_i$, we also have $\mu^{\el}_{\low}(\calA,D)=0$ almost surely.
\end{proof}

\begin{theorem}[Restatement of \cref{thm:finite-hallucination-impossibility}]
    There exists a language family $\calL$ such that for every deadline function $D:\bbN\to\bbN$ and every eventually consistent generator $\calA$,
    \[
        \inf_{K\in\calL}\inf_{E\in\calE(K)}
        \liminf_{i\to\infty}
        \mu^{\el}_{i}(\calA(E),K,D) = 0.
    \]
    Consequently, for every deadline function $D$ and every eventually consistent generator $\calA$, one necessarily has $\mu^{\el}_{\low}(\calA,D)=0$.
\end{theorem}

\begin{proof}
    This follows from observing that eventual consistency implies the condition in \cref{lem:deadline-hallucination-impossibility}.
\end{proof}

\begin{theorem}[Restatement of \cref{thm:linear-deadline-impossibility}]
    There exists a language family $\calL$ such that for any deadline function $D(i)=O(i)$, every generator $\calA$ with vanishing hallucination rate satisfies
    \[
        \inf_{K\in\calL}\inf_{E\in\calE(K)}
        \liminf_{i\to\infty}
        \mu^{\el}_{i}(\calA(E),K,D) = 0.
    \]
    Consequently, for every deadline function $D(i)=O(i)$ and every generator $\calA$ with vanishing hallucination rate, one necessarily has $\mu^{\el}_{\low}(\calA,D)=0$.
\end{theorem}

\begin{proof}
    This follows immediately from \cref{lem:deadline-hallucination-impossibility} by taking $D(i)=O(i)$.
\end{proof}

\section{A randomized algorithm with optimal lower density guarantees}\label{app:positive-results}

In this section discuss an algorithm capable of reducing obtaining element-wise density guarantees to obtaining density guarantees in the limit. We will in fact do this for a general data model and show that our results therefore extend far beyond the case where the adversary totally enumerates the target language to the generator. We start by defining the more general framework we consider.

We consider the same two-player game between an adversary and a generator where both players have access to $\calL$ and $U$. The adversary first selects a target language $K \in \calL$ and generates an observation sequence $X = (x_1, x_2, \ldots)$, where $X \in \mathfrak{D}(K)$ and $\mathfrak{D} : \calL \to \calP(U^\bbN)$ is a data model specifying the set of admissible observation sequences for each target language. The precise structure of $\mathfrak{D}(K)$ depends on the model and we defer formal instantiations of these models to later in this section.

The generator observes the sequence $X$ sequentially and after observing the prefix $x_{1:t}$, it outputs an element $o_t \in U$ that has not previously been output. Let $\calA(X) = (o_1, o_2, \ldots)$ denote the ordered output sequence of the generator. 

The generator’s objective, as before, is to maximize the density of its output within the target language. Formally, we define the upper and lower element-wise densities as
\begin{align*}
    \mu_{\up}^{\el}(\calA, D)
    &= \inf_{K\in\calL}\inf_{X\in\mathfrak{D}(K)} \limsup_{i \to \infty} \mu^{\el}_i(\calA(X), K, D)\\
    \mu^{\el}_{\low}(\calA, D)
    &= \inf_{K\in\calL}\inf_{X\in\mathfrak{D}(K)} \liminf_{i \to \infty} \mu^{\el}_i(\calA(X), K, D).
\end{align*}
and analogously upper and lower prefix-wise densities as,
\begin{align*}
    \mu^\pfx_{\up}(\calA, D)
    &= \inf_{K\in\calL}\inf_{X\in\mathfrak{D}(K)} \limsup_{i \to \infty} \mu^\pfx_i(\calA(X), K, D) \\
    \mu^\pfx_{\low}(\calA, D)
    &= \inf_{K\in\calL}\inf_{X\in\mathfrak{D}(K)} \liminf_{i \to \infty} \mu^\pfx_i(\calA(X), K, D),
\end{align*}
The framework above is intentionally agnostic to how the adversary generates the observation sequence. We now describe several natural instantiations.

\textbf{Total Enumeration.}
We recover the classical model of \cite{KM}, where the adversary reveals only elements of the target language. Formally, for each $K\in\calL$, define $\mathfrak{D}_{\mathrm{enum}}(K)$ to be the set of all sequences $X$ such that $x_t \in K$ for every $t$, and every element of $K$ appears at least once in $X$. Thus, $X$ is a total enumeration of $K$, potentially with repetitions. This recovers the standard setting in which the generator receives only truthful observations and eventually sees the entire target language.

\textbf{Partial Enumeration.}
A natural relaxation is to allow the adversary to reveal only part of the target language. For each $\alpha\in[0,1]$ and $K\in \calL$, define the following data model first introduced in \cite{KW2} 
\[
    \mathfrak{D}_{\mathrm{part},\alpha}(K) := \left\{X : x_t\in K \,\,\forall t,\  |X|=\infty,\  \mu_{\low}(X,K,\infty)\ge \alpha \right\}.
\]
Equivalently, for every $K\in \calL$, the sequence $X\in \mathfrak{D}_{\mathrm{part},\alpha}(K)$ is an enumeration, with possible repetitions, of an infinite subset
of $K$ whose lower density in $K$ is at least $\alpha$. In this model, the generator
receives only truthful observations, but may never observe large portions of the target
language. The parameter $\alpha$ controls how much of the target language must be
revealed in the limit.

\textbf{Finite Contamination.}
Another natural relaxation is to allow the adversary both to omit finitely many target strings
and to insert finitely many non-target strings. For each $K\in\calL$, following
\cite{mehrotra2025languagegenerationinfinitecontamination}, define
\[
    \mathfrak{D}_{\mathrm{fin\text{-}cont}}(K)
    :=
    \left\{
        X :
        |X|=\infty,\ 
        |X\setminus K|<\infty,\ 
        |K\setminus X|<\infty
    \right\}.
\]
Equivalently, $X\in\mathfrak{D}_{\mathrm{fin\text{-}cont}}(K)$
is an enumeration, with possible repetitions, of an infinite set whose symmetric difference
with $K$ is finite. In this model, the generator may observe finitely many false positives
and may permanently miss finitely many elements of the target language, but all but finitely
many observed strings are truthful and all but finitely many target strings are eventually
observed.

\textbf{Hybrid Models.}
The framework also accommodates mixtures of the above behaviors. For example, one may consider sequences that enumerate only a subset $S\subseteq K$ while also allowing a bounded on contamination. Such models interpolate between partial information and adversarial corruption, and are useful for studying the robustness of generation guarantees beyond the idealized enumeration setting.

\subsection{Reduction Between Density Metrics}\label{app:reduction}

In this section we provide a proof of the reduction from achieving lower density under the element-wise metric to achieving it under the prefix-wise metric. We start with the lemma statement and break it down into subproblems as per the discussion in \cref{sec:reduction}. We will prove the lemma at the end of this section. Note this reduction is data model agnostic and therefore generalizes to the framework above.

\begin{lemma}[Restatement of \cref{lem:reduction} for General Data Models]\label{lem:general-reduction}
    Fix any density $\rho \in [0,1]$ and feasible profile $(D_{\el},H_{\el})$. There exists a feasible profile $(D_\pfx,H_\pfx)$ such that, for every countable $\calL$, any generator $\calA$ satisfying $\mu^\pfx_{\low}(\calA,D_\pfx)\geq \rho$ and
    \[
        \sup_{K\in\calL}\sup_{X\in\mathfrak{D}(K)} H_{\calA}^{\calL,K,X}(t)\leq H_\pfx(t)
    \]
    for all sufficiently large $t$, also satisfies $\mu^{\el}_{\low}(\calA,D_{\el})\geq \rho$ and
    \[
        \sup_{K\in\calL}\sup_{X\in\mathfrak{D}(K)} H_{\calA}^{\calL,K,X}(t)\leq H_{\el}(t)
    \]
    for all sufficiently large $t$.
\end{lemma}

\begin{proof}
    Apply \cref{lem:pfx-window-construction-time} to obtain $D_\pfx$ such that we have $r(i)=o(i)$ and $\tau(t)/t=o(H_{\el}(t))$. Then apply \cref{lem:pfx-hallucination-construction-time} to obtain $H_\pfx$ with $H_\pfx(t)=\omega(D_\pfx^{-1}(t)/t)$ and $H_\pfx(t)\leq H_{\el}(t)$ for all sufficiently large $t$. Thus $(D_\pfx,H_\pfx)$ is a feasible profile.

    Now fix any countable $\calL$, and let $\calA$ be any generator satisfying $\mu^\pfx_{\low}(\calA,D_\pfx)\geq \rho$ and
    \[
        \sup_{K\in\calL}\sup_{X\in\mathfrak{D}(K)}H_{\calA}^{\calL,K,X}(t)\leq H_\pfx(t)
    \]
    for all sufficiently large $t$. Since $H_\pfx(t)\leq H_{\el}(t)$ for all sufficiently large $t$, we immediately get
    \[
        \sup_{K\in\calL}\sup_{X\in\mathfrak{D}(K)}H_{\calA}^{\calL,K,X}(t)\leq H_{\el}(t)
    \]
    for all sufficiently large $t$.

    It remains to prove the density guarantee. Fix $K\in\calL$ and $X\in\mathfrak{D}(K)$. For any $i$, every element among the first $i$ elements of $K$ that is generated by time $D_\pfx(i)$ is generated before its element-wise deadline under $D_{\el}$, except possibly those among the first $r(i)$ elements. Indeed, if $j>r(i)$, then $D_{\el}(j)\geq D_\pfx(i)$ by definition of $r(i)$. Therefore
    \[
        \mu_i^{\el}(\calA(X),K,D_{\el})
        \geq
        \mu_i^\pfx(\calA(X),K,D_\pfx)-\frac{r(i)}{i}.
    \]
    Taking lower limits and using $r(i)=o(i)$ gives
    \[
        \mu^{\el}_{\low}(\calA,D_{\el})
        \geq
        \mu^\pfx_{\low}(\calA,D_\pfx)
        \geq
        \rho,
    \]
    which completes the proof.
\end{proof}

\begin{lemma}\label{lem:countable-diagonal-domination}
    Let $f_1,f_2,\ldots:\bbN\to\bbN$ be monotone nondecreasing functions satisfying $f_m(i)=o(i)$ for every fixed $m\in\bbN$. Then there exists a monotone nondecreasing function $s:\bbN\to\bbN$ such that $s(i)\to\infty$, $s(i)=o(i)$, and, for every fixed $m\in\bbN$, $s(i)\geq f_m(i)$ for all sufficiently large $i$.
\end{lemma}

\begin{proof}
    For each $k\in\bbN$, since $f_m(i)=o(i)$ for every $m\leq k$, choose $N_k$ sufficiently large so that $N_1=1$, $N_k\geq k^2$, $N_{k+1}>N_k$, and $f_m(i)\leq i/k$ for every $m\leq k$ and every $i\geq N_k$. For $i\in[N_k,N_{k+1})$, define $s(i):=\max\{k,f_1(i),\ldots,f_k(i)\}$. Then $s$ is monotone nondecreasing. Moreover, for $i\in[N_k,N_{k+1})$ we have $s(i)\geq k$, so $s(i)\to\infty$. Also $s(i)\leq \max\{k,i/k\}$, and since $i\geq N_k\geq k^2$, we have $s(i)/i\leq \max\{k/i,1/k\}\leq 1/k$. Since $k\to\infty$, it follows that $s(i)=o(i)$. Finally, for every fixed $m$, whenever $i\geq N_m$, the interval containing $i$ has index at least $m$, and hence $s(i)\geq f_m(i)$.
\end{proof}

\begin{lemma}\label{lem:pfx-window-construction-time}
    Let $(D_{\el},H_{\el})$ be a feasible profile. Then there exists a monotone nondecreasing super-linear function $D_\pfx:\bbN\to\bbN$ such that, defining
    \[
        r(i):=\max\{j\leq i:D_{\el}(j)<D_\pfx(i)\}
        \qquad\text{and}\qquad
        \tau(t):=\min\{n\in\bbN:D_\pfx(n)> t\},
    \]
    we have $r(i)=o(i)$ and $\tau(t)/t=o(H_{\el}(t))$.
\end{lemma}

\begin{proof}
    Since $D_{\el}$ is super-linear, there exists a monotone nondecreasing function $g:\bbN\to\bbN$ such that $g(i)=o(i)$ and $D_{\el}(g(i))/i\to\infty$. Now, for each $m\in\bbN$, define
    \[
        q_m(n):=1+\max\left\{D_{\el}^{-1}(t):\left\lfloor \frac{tH_{\el}(t)}{m}\right\rfloor\leq n\right\},
    \]
    with the convention that the maximum of the empty set is $0$. Since $H_{\el}(t)=\omega(D_{\el}^{-1}(t)/t)$, we have $D_{\el}^{-1}(t)=o(tH_{\el}(t))$. Hence, for every fixed $m$, $q_m(n)=o(n)$.

    Apply \cref{lem:countable-diagonal-domination} to the countable family $\{g,q_1,q_2,\ldots\}$. Thus, there exists a monotone non-decreasing function $s:\bbN\to\bbN$ such that $s(i)\to\infty$, $s(i)=o(i)$, and $s$ eventually dominates $g$ and every $q_m$. Now, define $D_\pfx(i):=D_{\el}(s(i))$. Since $s$ and $D_{\el}$ are monotone nondecreasing, $D_\pfx$ is monotone nondecreasing. Moreover, since $s$ eventually dominates $g$, we have $D_\pfx(i)/i = D_{\el}(s(i))/i \geq D_{\el}(g(i))/i \to\infty$, so $D_\pfx$ is super-linear.

    We next show that $r(i)=o(i)$. For any $j\in \bbN$, if $D_{\el}(j)<D_\pfx(i)=D_{\el}(s(i))$, then by monotonicity of $D_{\el}$ we have $j<s(i)$. Therefore, $r(i)\leq s(i)=o(i)$.

    Finally, we prove $\tau(t)/t=o(H_{\el}(t))$. Fix any $\alpha>0$, and choose $m\in\bbN$ such that $1/m\leq \alpha$. Since $s$ eventually dominates $q_m$, for all sufficiently large $t$, $s(\lfloor \alpha tH_{\el}(t)\rfloor) \geq q_m(\lfloor \alpha tH_{\el}(t)\rfloor) > D_{\el}^{-1}(t)$. Therefore $D_\pfx(\lfloor \alpha tH_{\el}(t)\rfloor) = D_{\el}(s(\lfloor \alpha tH_{\el}(t)\rfloor)) > t$. Hence $\tau(t)\leq \lfloor \alpha tH_{\el}(t)\rfloor$ for all sufficiently large $t$. Since $\alpha>0$ was arbitrary, $\tau(t)/t=o(H_{\el}(t))$.
\end{proof}

\begin{lemma}\label{lem:pfx-hallucination-construction-time}
    Let $D_\pfx$ be monotone nondecreasing and super-linear, and define $\tau(t)$ as in \cref{lem:pfx-window-construction-time}. Suppose $H_{\el}$ is a vanishing rate satisfying $\tau(t)/t=o(H_{\el}(t))$. Then there exists a vanishing rate $H_\pfx$ such that $H_\pfx(t)=\omega(D_\pfx^{-1}(t)/t)$ and $H_\pfx(t)\leq H_{\el}(t)$ for all sufficiently large $t$.
\end{lemma}

\begin{proof}
    Since $\tau(t)/t=o(H_{\el}(t))$, $tH_{\el}(t)/\tau(t)\to\infty$. Now, define $a(t):=\inf_{u\geq t} uH_{\el}(u)/\tau(u)$. Then $a$ is monotone nondecreasing and $a(t)\to\infty$.

    Define $H_\pfx(t):=H_{\el}(t)/\sqrt{a(t)}$ for all sufficiently large $t$. Then, clearly we have $H_\pfx(t)\to 0$ and $H_\pfx(t)\leq H_{\el}(t)$. Finally, since $a(t)\leq tH_{\el}(t)/\tau(t)$, we have
    \[
        \frac{H_\pfx(t)}{\tau(t)/t}
        =
        \frac{tH_{\el}(t)}{\tau(t)\sqrt{a(t)}}
        \geq
        \sqrt{a(t)}
        \to\infty.
    \]
    Thus $H_\pfx(t)=\omega(\tau(t)/t)$. Since $D_\pfx^{-1}(t)\leq \tau(t)$, this implies $H_\pfx(t)=\omega(D_\pfx^{-1}(t)/t)$.
\end{proof}

\subsection{Achieving Lower Density Under the Prefix-Wise Metric}\label{app:algorithm-guarantee}

In this section discuss an algorithm capable of reducing obtaining prefix-wise density guarantees to obtaining density guarantees in the limit. We restate the lemma statement and prove the result. Then, below, we provide a couple key ingredients necessary for the proof. Note that although we state and prove the lemma for a deterministic eventually consistent blackbox generator, the argument extends straightforwardly to randomized blackbox generators as well, where the guarantees hold almost surely. More generally, if $\calB$ is not eventually consistent but instead has worst-case hallucination rate $H_{\calB}^{\calL,K,X}(t)$ satisfying $H_{\calB}^{\calL,K,X}(t)=\omega(D^{-1}(t)/t)$, then the collection of feasible $(D,H)$ profiles may itself depend on this rate. In particular, the achievable profiles yielded by the reduction are constrained by the hallucination behavior of the underlying blackbox generator.

\begin{lemma}[Restatement of \cref{lem:algorithm-guarantee} for General Data Models]\label{lem:general-algorithm-guarantee}
    Fix any data model $\mathfrak{D}$, any density $\rho\in [0,1]$, and any feasible profile $(D,H)$. For any deterministic eventually consistent generator $\calB$ such that $\mu_{\low}(\calB,\infty)=\rho$, there exists a randomized algorithm $\calA=\calA_\calB$ achieving \smash{$\mu^{\pfx}_{\low}(\calA,D)\geq\rho$} almost surely and for all countable collections $\calL$
    \[
        \sup_{K\in\calL}\sup_{X\in\mathfrak{D}(K)}H_{\calA}^{\calL,K,X}(t)\le H(t)
    \]
    for all sufficiently large $t$. Furthermore, $\calA$ only makes black-box calls to $\calB$.
\end{lemma}

\begin{proof}
    Fix any collection $\calL$, any $K\in\calL$, and any admissible sequence $X\in\mathfrak{D}(K)$. We analyze the algorithm in \cref{alg:speculative-blackbox-generation-general} on the input sequence $X$.

    We first justify the existence of the auxiliary function $C$. Define
    \[
        a(m):=D(m)-D(m-1)
        \qquad \text{and} \qquad
        b(m):=\inf_{t:\,m(t)\ge m} \frac{tH(t)}{m}.
    \]
    Since $D$ is super-linear and discrete convex, we have $a(m)\to\infty$. We also claim that $b(m)\to\infty$. Indeed, since $tH(t)=\omega(D^{-1}(t))$, for every $M>0$ there exists $T_M$ such that for all $t\ge T_M$, $tH(t)\ge M D^{-1}(t)$. Now, if $m(t)\ge m$, then $t>D(m-1)$ and hence $D^{-1}(t)\ge m-1$. Therefore, for all sufficiently large $m$, $\frac{tH(t)}{m} \ge M\frac{m-1}{m}$ uniformly over all $t$ satisfying $m(t)\ge m$. Taking the infimum over such $t$ yields $b(m)\ge M/2$ for all sufficiently large $m$. Since $M>0$ was arbitrary, $b(m)\to\infty$.
    
    Applying \cref{lem:slow-divergent-minorant}, there exists a monotone nondecreasing function $C:\bbN\to\bbN$ such that $C(m)\to\infty$, $C(m)=o(a(m))$, and $C(m)=o(b(m))$. The first condition implies $\delta_m\to0$. Moreover, since $C$ is monotone nondecreasing,
    \[
        \sum_{k\le m(t)} C(k)
        \le
        m(t)C(m(t)).
    \]
    By definition of $b$, $b(m(t)) \le \frac{tH(t)}{m(t)}$.
    Thus, since $C(m)=o(b(m))$, $\sum_{k\le m(t)} C(k) = o(tH(t))$.

    Now, we prove the hallucination guarantee. Since $\calB$ is deterministic and eventually consistent on $X$, it produces only finitely many outputs outside $K$; denote this finite quantity by $N_{\calB}(X,K)$. Recall $m(t)=\min\{m\in\bbN:t\le D(m)\}$ is the index of the epoch containing time $t$. Thus,
    \begin{align*}
        \bbE\big[|\calA(X)_t\setminus K|\big]
        &\le
        N_{\calB}(X,K)
        +
        \sum_{m\le m(t)}\delta_m(D(m)-D(m-1)) \\
        &\le
        N_{\calB}(X,K)
        +
        10\sum_{m\le m(t)}C(m).
    \end{align*}
    By the choice of $C$, the second term is $o(tH(t))$ and therefore there exists some time $t_1$ such that for all $t\geq t_1$ the second term is at most $tH(t)/2$. Moreover, since $tH(t)=\omega(D^{-1}(t))$ and $D^{-1}(t)\to\infty$, we have $tH(t)\to\infty$. Since $N_{\calB}(X,K)<\infty$, there exists some $t_2$ such that for all $t\geq t_2$ the first term is also at most $tH(t)/2$. Therefore, for any $t\geq \max\{t_1,t_2\}$, we have
    \[
        H_{\calA}^{\calL,K,X}(t)
        =
        \bbE\left[\frac{|\calA(X)_t\setminus K|}{t}\right]
        \le H(t).
    \]

    It remains to prove the density guarantee. We begin by arguing that after some finite time the generator has almost surely speculated $i$ elements of $K$ by time $D(i)$. Since $K=L^j$ for some $j\in \bbN$, $C(m)\to\infty$ and $\delta_m\to 0$, there exists $m_K$ such that for all $m\ge m_K$, we have $j\le C(m)$ and $\delta_m=\frac{10C(m)}{D(m)-D(m-1)}$. Moreover, since $m_K$ is fixed, there also exists $i_K$ such that for all $i\ge i_K$, $10(i-m_K+1)\ge 5i$.
    
    For each round $t$ in epoch $m$, let $Y_t$ be the indicator that the generator speculates from $K$ at time $t$. Then for all $m\ge m_K$ and all $t$ in epoch $m$,
    \[
        \Pr(Y_t=1)
        =
        \frac{\delta_m}{C(m)}
        =
        \frac{10}{D(m)-D(m-1)}.
    \]
    Hence, by linearity of expectation and our choice of $i_K$, for all $i\ge i_K$,
    \[
        \bbE\left[\sum_{t=1}^{D(i)}Y_t\right]
        \geq \sum_{m=m_K}^i \frac{10}{D(m)-D(m-1)}\cdot(D(m)-D(m-1))
        =
        10(i-m_K+1) \geq 5i.
    \]
    Since the variables $\{Y_t\}_{t\le D(i)}$ are independent, a Chernoff bound yields a constant $c>0$ such that for all $i\ge i_K$,
    \[
        \Pr\!\left(\sum_{t=1}^{D(i)}Y_t<i\right)
        \le e^{-ci}.
    \]
    Thus, by the first Borel--Cantelli lemma, almost surely, the event that the number of speculations from $K$ by time $D(i)$ is less than $i$ occurs finitely often and therefore there exists an index $i^\star$ such that for all $i\geq i^\star$ the number of speculations from $K$ by time $D(i)$ is at least $i$.

    Now fix any $\epsilon>0$. For each sufficiently large $i$, define $t_i$ to be the first timestep such that every element of $K_i$ has appeared either in the output of the blackbox generator $\calB$ or in the adversary's sequence. Since $\mu_{\low}(\calB,\infty)=\rho$, there exists $i_1(\epsilon)$ such that for all $i\ge i_1(\epsilon)$, by time $t_i$ the blackbox generator $\calB$ has output at least $(\rho-\tfrac{\epsilon}{2})i$ elements of $K_i$. Since $\delta_m\to 0$, there exists $m(\epsilon)$ such that $\delta_m\le \tfrac{\epsilon}{2}$ for all $m\ge m(\epsilon)$. Since $D(m(\epsilon)-1)$ is finite, there exists $i_2(\epsilon)$ such that for all $i\ge i_2(\epsilon)$, $D(m(\epsilon)-1)\leq \epsilon i/2$.

    Fix any $i\ge \max\{i_1(\epsilon),i_2(\epsilon),i^\star\}$. It suffices to show that there exists a constant $c_\epsilon>0$ such that
    \[
        \Pr\left(\mu^{\pfx}_i(\calA(X),K,D) < \rho-2\epsilon\right)\leq \exp(-c_\epsilon i).
    \]
    Indeed, once this is established, we have
    \[
        \sum_{i=1}^\infty \Pr\left(\mu^{\pfx}_i(\calA(X),K,D) < \rho-2\epsilon\right)<\infty,
    \]
    and by the first Borel--Cantelli lemma, with probability $1$, the event
    $\mu^{\pfx}_i(\calA(X),K,D) < \rho-2\epsilon$ occurs only finitely often. Equivalently, with probability $1$, for all sufficiently large $i$, we have $\mu^{\pfx}_i(\calA(X),K,D)\geq \rho-2\epsilon$.
    Since $\epsilon>0$ was arbitrary, it follows that $\mu^{\pfx}_{\low}(\calA,D) \geq \rho$ almost surely.

    We now consider two cases. First suppose that $t_i\geq D(i)$. By the definition of $t_i$, by time $t_i$ the blackbox has output at least $(\rho-\tfrac{\epsilon}{2})i$ credited elements of $K_i$. Since credit can only be obtained for elements of $K_i$ not already claimed by the adversary, the adversary can have claimed at most $(1-\rho+\tfrac{\epsilon}{2})i$ elements of $K_i$ by time $t_i$. Since $D(i)\leq t_i$, the adversary can have claimed at most $(1-\rho+\tfrac{\epsilon}{2})i$ elements of $K_i$ by time $D(i)$. Therefore at least $(\rho-\tfrac{\epsilon}{2})i$ elements of $K_i$ are still unclaimed by the adversary at time $D(i)$.

    Since $i\geq i^\star$, the generator makes at least $i$ speculations from $K$ by time $D(i)$. Each speculation from $K$ outputs the earliest element of $K$ unused by both the generator and the adversary. Thus, among the elements of $K_i$ that are unclaimed by the adversary at time $D(i)$, any element not output through speculation must already have been output by the generator through the blackbox. These elements already count toward $\calA(X)_{D(i)}\cap K_i$, and therefore
    \[
        |\calA(X)_{D(i)}\cap K_i|
        \geq
        \left(\rho-\frac{\epsilon}{2}\right)i
        \geq
        (\rho-2\epsilon)i.
    \]

    It remains to consider the case $t_i<D(i)$. Let $U_i$ be a set of elements of $K_i$ output by $\calB$ by time $t_i$, so that $|U_i|\geq \left(\rho-\epsilon/2\right)i$.
    For each $u\in U_i$, let $\tau(u)$ denote the first round $t\leq t_i$ at which $\calB$ outputs $u$. Partition $U_i$ into
    \[
        U_i^{Spec}:=\{u\in U_i : u\in \calA(X)_{\tau(u)-1}\}
        \qquad \text{and} \qquad
        U_i^{BB}:=U_i\setminus U_i^{Spec}.
    \]
    By definition, if $u\in U_i^{Spec}$, then $u$ already appears among the generator's first $\tau(u)-1$ outputs, and hence $u\in \calA(X)_{D(i)}\cap K_i$. Let $T_i:=\{\tau(u):u\in U_i^{BB}\}$. Since each $u\in U_i^{BB}$ has a unique first occurrence time, we have $|T_i|=|U_i^{BB}|$. Moreover, if $t\in T_i$ and the generator retains the blackbox output at time $t$, then $o_t=u\in K_i$ and this element is not already in $\calA(X)_{t-1}$ by definition of $U_i^{BB}$. Therefore, each such round contributes a distinct element of $K_i$ to the generator's output.

    We split $T_i$ into early and late rounds in order to isolate the finite prefix of epochs where the retention probabilities may still be large. Define
    \[
        T_i^{late}:=\{t\in T_i : m(t)\geq m(\epsilon)\},
        \qquad
        T_i^{early}:=T_i\setminus T_i^{late}.
    \]
    Since $T_i^{early}$ is contained in the first $m(\epsilon)-1$ epochs, by our choice of $i_2(\epsilon)$ we have $|T_i^{early}| \leq D(m(\epsilon)-1) \leq \epsilon i/2$.

    Now, for each $t\leq t_i$, define the indicator random variable $Z_t:=\ind\{t\in T_i^{late}\}\cdot \ind\{\decision_t=0\}$, which is $1$ exactly when a late blackbox output from $K_i$ is retained by the generator. Since the contributions counted by $\sum_{t=1}^{t_i}Z_t$ correspond to distinct elements of $K_i$ and are disjoint from $U_i^{Spec}$, we have
    \begin{align*}
        |\calA(X)_{D(i)}\cap K_i|
        \geq
        |\calA(X)_{t_i}\cap K_i|
        \geq
        |U_i^{Spec}|+\sum_{t=1}^{t_i} Z_t \\
        \implies \mu^{\pfx}_i(\calA(X),K,D)\geq \frac{1}{i}\left(|U_i^{Spec}|+\sum_{t=1}^{t_i} Z_t\right).
    \end{align*}
    Let $\calF_{t-1}$ denote the natural filtration containing the history up to immediately before the retention coin $\decision_t$ is flipped. In particular, the blackbox output at time $t$ has already been determined at this point, so the event $\{t\in T_i^{late}\}$ is $\calF_{t-1}$-measurable, while $\decision_t$ is sampled independently of $\calF_{t-1}$. Therefore,
    \[
        \bbE[Z_t\mid \calF_{t-1}]
        =
        \ind\{t\in T_i^{late}\}(1-\delta_{m(t)})
        \geq
        \left(1-\frac{\epsilon}{2}\right)\ind\{t\in T_i^{late}\},
    \]
    where the inequality follows because $t\in T_i^{late}$ implies $m(t)\geq m(\epsilon)$.

    Define $M_t:=Z_t-\bbE[Z_t\mid\calF_{t-1}]$. Then $\{M_t\}_{t=1}^{t_i}$ is a martingale difference sequence with respect to $\{\calF_t\}$, with $|M_t|\leq 1$ almost surely. Moreover,
    \[
        \Var(M_t\mid\calF_{t-1})
        =
        \Var(Z_t\mid\calF_{t-1})
        \leq
        \bbE[Z_t\mid\calF_{t-1}]
        \leq
        \ind\{t\in T_i^{late}\}.
    \]
    Thus the predictable quadratic variation satisfies
    \[
        \sum_{t=1}^{t_i}\Var(M_t\mid\calF_{t-1})
        \leq
        |T_i^{late}|
        \leq
        |T_i|
        =
        |U_i^{BB}|
        \leq
        i.
    \]
    Applying Freedman's inequality to the martingale difference sequence $\{-M_t\}_{t=1}^{t_i}$ with $R=1$, $\eta=\epsilon i/2$, and $\sigma^2=i$, we obtain
    \[
        \Pr\left(
            \sum_{t=1}^{t_i} M_t \leq -\frac{\epsilon}{2}i
        \right)
        \leq
        \exp(-c_\epsilon i)
    \]
    for some constant $c_\epsilon>0$. On the complementary event,
    \[
        \sum_{t=1}^{t_i} Z_t
        \geq
        \sum_{t=1}^{t_i}\bbE[Z_t\mid\calF_{t-1}]
        -
        \frac{\epsilon}{2}i
        \geq
        \left(1-\frac{\epsilon}{2}\right)|T_i^{late}|
        -
        \frac{\epsilon}{2}i.
    \]
    Since $|T_i^{late}|=|U_i^{BB}|-|T_i^{early}|\geq |U_i^{BB}|-D(m(\epsilon)-1)$, we obtain
    \begin{align*}
        |U_i^{Spec}|+\sum_{t=1}^{t_i}Z_t
        &\geq
        |U_i^{Spec}|+\left(1-\frac{\epsilon}{2}\right)(|U_i^{BB}|-D(m(\epsilon)-1))-\frac{\epsilon}{2}i \\
        &=
        |U_i|-\frac{\epsilon}{2}|U_i^{BB}|-\left(1-\frac{\epsilon}{2}\right)D(m(\epsilon)-1)-\frac{\epsilon}{2}i \\
        &\geq 
        |U_i|-\epsilon i-D(m(\epsilon)-1) \\
        &\geq
        \left(\rho-\frac{3\epsilon}{2}\right)i-D(m(\epsilon)-1) \\
        &\geq
        (\rho-2\epsilon)i,
    \end{align*}
    where the second inequality follows from $|U_i^{BB}|\leq i$, the third follows from $|U_i|\geq (\rho-\tfrac{\epsilon}{2})i$, and the final from $D(m(\epsilon)-1)\leq \tfrac{\epsilon}{2}i$. Therefore, in the case $t_i<D(i)$,
    \[
        \Pr\left(\mu_i^{\pfx}(\calA(X),K,D)<\rho-2\epsilon\right)
        \leq
        \exp(-c_\epsilon i).
    \]
    Combining this with the deterministic guarantee in the case $t_i\geq D(i)$ proves that the probability of not achieving sufficient density is exponentially small. Taking $\epsilon$ to $0$, the Borel--Cantelli argument above implies that $\mu^{\pfx}_{\low}(\calA,D)\geq \rho$ almost surely.
\end{proof}

\begin{lemma}\label{lem:slow-divergent-minorant}
    Let $a,b:\bbN\to\bbR_+$ satisfy $a(i)\to\infty$ and $b(i)\to\infty$. Then there exists a monotone nondecreasing function $s:\bbN\to\bbN$ such that $s(i)\to\infty$, $s(i)=o(a(i))$, and $s(i)=o(b(i))$.
\end{lemma}

\begin{proof}
    For each $k\in\bbN$, since $a(i)\to\infty$ and $b(i)\to\infty$, choose $N_k$ sufficiently large so that $N_1=1$, $N_{k+1}>N_k$, and
    \[
        a(i)\geq k^2
        \qquad\text{and}\qquad
        b(i)\geq k^2
    \]
    for every $i\geq N_k$. For $i\in[N_k,N_{k+1})$, define $s(i):=k$. Then $s$ is monotone nondecreasing and $s(i)\to\infty$. Moreover, for $i\in[N_k,N_{k+1})$, we have $s(i)=k$, while $a(i)\geq k^2$ and $b(i)\geq k^2$. Hence
    \[
        \frac{s(i)}{a(i)}\leq \frac{1}{k}
        \qquad\text{and}\qquad
        \frac{s(i)}{b(i)}\leq \frac{1}{k}.
    \]
    Since $k\to\infty$, it follows that $s(i)=o(a(i))$ and $s(i)=o(b(i))$.
\end{proof}

\begin{theorem}[Freedman's Inequality i.e. Theorem 1.6 in \cite{Freedman}]\label{thm:freedman}
    Let $\{M_t\}_{t=1}^T$ be a martingale difference sequence adapted to a filtration $\{\calF_t\}_{t=0}^T$ such that $|M_t|\leq R$ almost surely for all $t$. Define
    \[
        W_k:=\sum_{t=1}^k \Var(M_t\mid\calF_{t-1}).
    \]
    Then, for all $\eta,\sigma^2>0$,
    \[
        \Pr\left(
            \exists k\leq T:
            \sum_{t=1}^k M_t\geq \eta
            \text{ and }
            W_k\leq \sigma^2
        \right)
        \leq
        \exp\left(
            -\frac{\eta^2/2}{\sigma^2+R\eta/3}
        \right).
    \]
\end{theorem}

\subsection{Proof of Main Theorem and Extensions to Other Data Models}\label{app:corollaries}

In this section we provide examples of how our results in \cref{lem:general-reduction} and \cref{lem:general-algorithm-guarantee} allow us to generalize to other data models. 

In fact, we will use this as an opportunity to formally prove \cref{thm:lower-density-guarantee} by proving a more general result for the partial enumeration model of \cite{KW2}. Recall the partial enumeration model from \cref{app:positive-results}. In \cite{KW2}, they show the following:

\begin{theorem}[Theorem 1.6 of \cite{KW2}]
\label{thm:kw-partial-enumeration}
    There exists a deterministic eventually consistent generator $\calB_{\mathrm{KW}}$ such that for every $\alpha\in[0,1]$, under the data model $\mathfrak{D}_{\mathrm{part},\alpha}$, $\mu_{\low}(\calB_{\mathrm{KW}},\infty)\ge \alpha/2$.
\end{theorem}

We will now use this to get the following timely generation guarantee.

\begin{theorem}[Timely Dense Generation from Partial Enumeration]
\label{thm:partial-enumeration-timely}
    Fix any $\alpha\in[0,1]$ and any feasible profile $(D,H)$. Under the data model
    $\mathfrak{D}_{\mathrm{part},\alpha}$, there exists a randomized generation
    algorithm $\calA=\calA_{\calB_{\mathrm{KW}}}$ such that $\mu^{\el}_{\low}(\calA,D)\ge \alpha/2$
    almost surely and, for every countable collection $\calL$,
    \[
        \sup_{K\in\calL}
        \sup_{X\in\mathfrak{D}_{\mathrm{part},\alpha}(K)}
        H_{\calA}^{\calL,K,X}(t)
        \le H(t)
    \]
    for all sufficiently large $t$.
\end{theorem}

\begin{proof}
    Fix $\alpha\in[0,1]$ and a feasible profile $(D,H)$. By \cref{thm:kw-partial-enumeration}, there exists a deterministic eventually consistent generator $\calB_{\mathrm{KW}}$ such that, for $\mathfrak{D}_{\mathrm{part},\alpha}$, $\mu_{\low}(\calB_{\mathrm{KW}},\infty)\ge \alpha/2$. Let $\rho_\alpha := \mu_{\low}(\calB_{\mathrm{KW}},\infty)\ge \alpha/2$. 
    
    Apply \cref{lem:general-reduction} with density parameter
    $\rho_\alpha$ and target feasible profile $(D,H)$. This yields a feasible profile $(D_\pfx,H_\pfx)$ such that any generator satisfying $\mu^\pfx_{\low}(\cdot,D_\pfx)\ge \rho_\alpha$ and hallucination rate bounded by $H_\pfx$ also satisfies $\mu^{\el}_{\low}(\cdot,D)\ge \rho_\alpha$ and hallucination rate bounded by $H$.

    Now apply \cref{lem:general-algorithm-guarantee} with data model $\mathfrak{D}=\mathfrak{D}_{\mathrm{part},\alpha}$, density parameter $\rho_\alpha$, feasible profile $(D_\pfx,H_\pfx)$, and
    black-box generator $\calB_{\mathrm{KW}}$. We obtain a randomized algorithm $\calA=\calA_{\calB_{\mathrm{KW}}}$ such that $\mu^\pfx_{\low}(\calA,D_\pfx)\ge \rho_\alpha$ almost surely and, for every countable collection $\calL$,
    \[
        \sup_{K\in\calL} \sup_{X\in\mathfrak{D}_{\mathrm{part},\alpha}(K)}
        H_{\calA}^{\calL,K,X}(t)
        \le H_\pfx(t)
    \]
    for all sufficiently large $t$.

    Therefore, by \cref{lem:general-reduction}, $\mu^{\el}_{\low}(\calA,D)\ge \rho_\alpha\ge \alpha/2$
    almost surely and
    \[
        \sup_{K\in\calL} \sup_{X\in\mathfrak{D}_{\mathrm{part},\alpha}(K)}
        H_{\calA}^{\calL,K,X}(t)
        \le H(t)
    \]
    for all sufficiently large $t$.
\end{proof}

We can now prove the main result in the main body of our paper.

\begin{proof}[Proof of \cref{thm:lower-density-guarantee}]
    Apply \cref{thm:partial-enumeration-timely} with $\alpha=1$.
\end{proof}

Another example of when we can directly apply our reduction is contamination framework of \cite{mehrotra2025languagegenerationinfinitecontamination}. We will provide an explicit example of the finite contamination data model; however, the reduction extends to the positive results for their other contamination models as well. In particular, exactly when they can achieve density $\rho$ in the limit under any one of their data models we can convert that into an algorithm that does so under any superlinear deadline and with any feasible hallucination rate.

The work in \cite{mehrotra2025languagegenerationinfinitecontamination} combined with the work in \cite{KW2} implies the following guarantee.

\begin{theorem}[Theorem 6.18 of \cite{mehrotra2025languagegenerationinfinitecontamination} + Theorem 1.6 of \cite{KW2}]\label{thm:finite-contamination-base}
    There exists a deterministic eventually consistent generator $\calB_{\mathrm{MVYZ}}$ such that, under the data model $\mathfrak{D}_{\mathrm{fin\text{-}cont}}$, $\mu_{\low}(\calB_{\mathrm{MVYZ}},\infty)\ge 1/2.$
\end{theorem}

With this, we get the following timely generation guarantee.

\begin{theorem}[Timely Dense Generation from Finite Contamination]
\label{thm:finite-contamination-timely}
    Fix any feasible profile $(D,H)$. Under the data model $\mathfrak{D}_{\mathrm{fin\text{-}cont}}$, there exists a randomized generation algorithm $\calA=\calA_{\calB_{\mathrm{MVYZ}}}$ such that $\mu^{\el}_{\low}(\calA,D)\ge 1/2$ almost surely and, for every countable collection $\calL$,
    \[
        \sup_{K\in\calL}
        \sup_{X\in\mathfrak{D}_{\mathrm{fin\text{-}cont}}(K)}
        H_{\calA}^{\calL,K,X}(t)
        \le H(t)
    \]
    for all sufficiently large $t$.
\end{theorem}

\begin{proof}
    Fix a feasible profile $(D,H)$. By \cref{thm:finite-contamination-base}, there exists a deterministic eventually consistent generator $\calB_{\mathrm{MVYZ}}$ such that, under $\mathfrak{D}_{\mathrm{fin\text{-}cont}}$, we have $\mu_{\low}(\calB_{\mathrm{MVYZ}},\infty)=1/2$. Let $\rho:=1/2$.

    Apply \cref{lem:general-reduction} with density parameter $\rho$ and target feasible profile $(D,H)$. This yields a feasible profile $(D_\pfx,H_\pfx)$ such that any generator satisfying $\mu^\pfx_{\low}(\cdot,D_\pfx)\ge \rho$ and hallucination rate bounded by $H_\pfx$ also satisfies $\mu^{\el}_{\low}(\cdot,D)\ge \rho$ and hallucination rate bounded by $H$.

    Now apply \cref{lem:general-algorithm-guarantee} with data model $\mathfrak{D}=\mathfrak{D}_{\mathrm{fin\text{-}cont}}$, density parameter $\rho$, feasible profile $(D_\pfx,H_\pfx)$, and black-box generator $\calB_{\mathrm{MVYZ}}$. We obtain a randomized algorithm $\calA=\calA_{\calB_{\mathrm{MVYZ}}}$ such that $\mu^\pfx_{\low}(\calA,D_\pfx)\ge \rho$ almost surely and, for every countable collection $\calL$,
    \[
        \sup_{K\in\calL}
        \sup_{X\in\mathfrak{D}_{\mathrm{fin\text{-}cont}}(K)}
        H_{\calA}^{\calL,K,X}(t)
        \le H_\pfx(t)
    \]
    for all sufficiently large $t$.

    Therefore, by \cref{lem:general-reduction}, we conclude that $\mu^{\el}_{\low}(\calA,D)\ge \rho=1/2$ almost surely and
    \[
        \sup_{K\in\calL}
        \sup_{X\in\mathfrak{D}_{\mathrm{fin\text{-}cont}}(K)}
        H_{\calA}^{\calL,K,X}(t)
        \le H(t)
    \]
    for all sufficiently large $t$.
\end{proof}

Our results can clearly port to many additional data models by virtue of its black-box nature and we defer exploration of these details to the full version of this paper, as well as to future work.

\section{Optimal linear-time upper density via eventually consistent generators}\label{app:upper-density}

In this section, we focus on the total enumeration data model and show there exists a consistent generation algorithm that achieves optimal upper density under the identity deadline function. 

We begin with a result from \cite{KW}. They show that for any collection of languages $\calL$, there exists a simple, deterministic algorithm $\mathsf{Accurate} = \{\mathsf{Accurate}_t : U^t \to \calL\}_{t=1}^{\infty}$ with the following property. For any target language $K \in \calL$ and any total enumeration $E \in \calE(K)$ of $K$, if we denote the output at time $t$ by $\hat{L}^t = \mathsf{Accurate}(x_{1:t})$, then:
\begin{enumerate}
    \item $\hat{L}^t = K$ for infinitely many $t$, and
    \item there exists a finite time $T$ such that for all $t \ge T$, $\hat{L}^t \subseteq K$.
\end{enumerate}
That is, the algorithm eventually only produces guesses that are consistent with the target language, and moreover, it guesses the target language correctly infinitely often.

Using this procedure, Kleinberg and Wei \cite{KW} construct a generation algorithm $\calA$ satisfying $\mu_{up}(\calA, \infty) = 1/2$, which can easily be shown to be optimal. We now show that a simple modification of this construction yields a stronger result: one can obtain an algorithm $\mathsf{GCG}$, given in \cref{alg:consistent-generation}, such that $\mu^{\el}_{up}(\mathsf{GCG}, D) = 1/2$ even under the identity deadline function $D(i) = i$.

\begin{algorithm}[h]
\caption{\textsf{G}eneration via \textsf{C}onsistent \textsf{G}uessing (\textsf{GCG})}
\label{alg:consistent-generation}
\begin{algorithmic}[1]
\State Initialize $m \gets 0$, $G \gets [\,]$, and $\hat{L}^0 \gets L^1$ \Comment{$G=(g_1,\dots,g_{|G|})$; $g_m$ is its $m$-th element}
\State Define $\alpha_m \gets 1/2 - 2^{-m}$ for all $m \in \bbN$

\For{$t = 1,2,\dots$}
    \State $\hat{L}^t \gets \mathsf{Accurate}(H^{t-1}, x_t)$
    
    \If{$\hat{L}^t \supsetneq \hat{L}^{t-1}$}
        \State Append $\hat{L}^t$ to $G$ 
    \EndIf
    
    \If{$m = 0$}
        \State Output $\mathsf{OnTimeUnused}(\hat{L}^t)$ \Comment{first element of $\hat{L}_t$ not yet late and not output/observed}
    \Else
        \State Output $\mathsf{OnTimeUnused}(g_m)$
    \EndIf
    
    \If{$|G| \ge m+1$ and $\mu^{\el}_t(\mathsf{GCG}(E), g_m, t) \ge \alpha_m$}
        \State $m \gets m+1$
    \EndIf
\EndFor
\end{algorithmic}
\end{algorithm}

The generator maintains an ordered list $G$ of candidate languages produced over time. Whenever $\mathsf{Accurate}$ outputs a guess that is a strict superset of its previous guess, this guess is appended to $G$. Intuitively, these strict superset transitions correspond to moments where the learner expands its current hypothesis and therefore may expose new elements worth generating from. Thus, at any time we may write $G = (g_1,\dots,g_{|G|})$, where $g_m$ denotes the $m$-th element whenever $m \le |G|$.

In addition to $G$, the generator maintains a stage index $m \in \bbN$ that determines from which candidate language it generates. When $m=0$, the generator is in an initialization phase and generates directly from the current guess $\hat{L}^t$ produced by $\mathsf{Accurate}$. Once $G$ becomes nonempty, the generator transitions to $m=1$ and subsequently generates from elements of $G$.

For any language $L \subseteq U$, let $\mathsf{OnTimeUnused}(L)$ denote the first element of $L$ (under the canonical ordering of $U$) that is not yet late (according to the deadline function) and has neither previously been output by the generator nor previously been observed from the adversary. At each round $t$, the generator outputs $\mathsf{OnTimeUnused}(\hat{L}^t)$ when $m=0$, and $\mathsf{OnTimeUnused}(g_m)$ when $m \ge 1$.

The progression of the stage index $m$ is governed by a sequence of thresholds $\{\alpha_m\}_{m\in\bbN}$ defined by $\alpha_m = 1/2 - 2^{-m}$. Whenever $m \ge 1$ and $G$ contains at least $m+1$ elements, the generator monitors its empirical density with respect to the current candidate $g_m$. If at time $t$ the condition $\mu^{\el}_t(\mathsf{GCG}(E), g_m, t) \ge \alpha_m$ is satisfied, then the generator increments $m \gets m+1$ and begins generating from the next candidate in $G$. In this way, the generator proceeds through the list $G$, spending sufficient time generating from each candidate before advancing.

We now state the formal guarantee achieved by \cref{alg:consistent-generation}, given below, and directly provide a full proof.

\begin{theorem}[Restatement of \cref{thm:total enumeration consistent timely generation}]
    The deterministic algorithm $\mathsf{GCG}$ is eventually consistent and satisfies $\mu^{\el}_{up}(\mathsf{GCG}, D) = 1/2$ with deadline function $D(i)=i$.
\end{theorem}

\begin{proof}
    We recall the guarantees of the algorithm $\mathsf{Accurate}$ from the main body. For any target language $K \in \calL$ and any total enumeration $E\in \calE(K)$, the sequence of guesses $\{\hat{L}^t\}_{t\in \bbN}$ satisfies that
    \begin{enumerate}
        \item $\hat{L}^t = K$ for infinitely many $t$, and
        \item there exists a finite time $T$ such that for all $t \geq T$, $\hat{L}^t \subseteq K$.
    \end{enumerate}
    
    The algorithm $\mathsf{GCG}$ maintains two objects: a list $G = (g_1,\dots,g_{|G|})$ of candidate languages which is initialized to be empty and a stage index $m \in \bbN$ which is initialized to be $0$. The list $G$ is constructed from the guesses output by $\mathsf{Accurate}$ as follows: whenever at some round $t$ we have $\hat{L}^t \supsetneq \hat{L}^{t-1}$, the guess $\hat{L}^t$ is appended to $G$. Thus, $G$ consists precisely of those guesses produced by $\mathsf{Accurate}$ that are strict supersets of its previous guess.
    
    At each round, if $m=0$, the generator outputs the first element of $\hat{L}^t$ that is both not yet late and has not previously been output by the generator or observed from the adversary. Otherwise, it outputs the first element of $g_m$ with this same property. The stage index $m$ is incremented once the generator has achieved the density threshold $\mu^{\el}_i(\mathsf{GCG}(E), g_m, i) \geq \alpha_m$, where $\alpha_m = \tfrac{1}{2} - 2^{-m}$.
    
    We analyze the behavior of the algorithm via three cases.
    
    \textbf{$G$ remains empty:} In this case, there never exists a time $t$ for which $\hat{L}^t \supsetneq \hat{L}^{t-1}$, as otherwise $G$ would be nonempty. Since there are infinitely many rounds $t$ for which $\hat{L}^t = K$, and since $T$ is finite, there must also be infinitely many rounds after time $T$ for which this is true. Furthermore, since for all $t \geq T$ we have $\hat{L}^t \subseteq K$, it cannot be the case that there exists some $t \geq T$ for which $\hat{L}^t \subsetneq K$, as otherwise there would exist some later round $t' > t$ for which $\hat{L}^{t'} = K$, implying $\hat{L}^{t'} \supsetneq \hat{L}^{t'-1}$, a contradiction. As such, we must have $\hat{L}^t = K$ for all $t \geq T$.
    
    It follows that for every round $t \geq T$, the generator outputs the earliest element of $K$ that is both not yet late and has not previously been output by the generator or observed from the adversary. We now show that this implies $\limsup_{i\to\infty}\mu^{\el}_i(\mathsf{GCG}(E),K,i)\geq \tfrac{1}{2}$.
    
    Now fix any $i > 2T$. By time $T$, the adversary has made at most $T$ outputs and hence, in the worst case, as many as $2T$ elements of $K$ can no longer be generated for credit by the generator. In particular, since the deadline function is $D(i)=i$, the generator will have already lost credit for the first $T$ elements of $K$ and the adversary may have generated $T$ more elements of $K \setminus K_T$ that can then no longer be generated by the generator for credit. Therefore, at time $T$ at least $i-2T$ elements of $K_i$ remain available to the generator to generate for credit.
    
    From time $T$ onward, the generator always outputs the earliest element of $K$ that is both not yet late and has not previously been generated by either player. Thus, among these $i-2T$ available elements, the generator greedily takes the earliest available element at each opportunity, while the adversary can remove at most one such element per round. It follows that, in the worst case, the adversary and generator alternate in generating elements from this set. Consequently, the generator generates at least $\left\lfloor \frac{i-2T}{2} \right\rfloor$ elements of $K_i$ on time by the time all elements of $K_i$ have been generated.
    
    Therefore, we have
    \[
        \mu^{\el}_i(\mathsf{GCG}(E),K,i) = \frac{1}{i}\sum_{j=1}^i \ind\{k_j \in \mathsf{GCG}(E)_j\} \geq \frac{\lfloor (i-2T)/2 \rfloor}{i}.
    \]
    Therefore, $\limsup_{i\to\infty}\mu^{\el}_i(\mathsf{GCG}(E),K,i)\geq \tfrac{1}{2}$.
    
    \textbf{$G$ grows indefinitely:} In this case, $\hat{L}^t$ does not converge to a fixed language, and the list $G$ grows without bound. We first claim that $K$ appears infinitely often in $G$. Indeed, since there are infinitely many rounds $t$ for which $\hat{L}^t = K$ and for all $t \geq T$ we have $\hat{L}^t \subseteq K$, it follows that whenever the sequence transitions from a strict subset of $K$ to $K$, we must have $\hat{L}^t \supsetneq \hat{L}^{t-1}$, and hence $K$ is appended to $G$. Since this occurs infinitely often, $K$ appears infinitely often in $G$.
    
    We now show that each stage $m$ terminates in finite time. Suppose for the sake of contradiction that stage $m$ did not terminate. Fix any $m$ and suppose the algorithm enters stage $m$ at some time $T_m$. Then for all $t\geq T_m$, round $t$ would be in stage $m$. From time $T_m$ onward, the generator outputs the earliest element of $g_m$ that is both not yet late and has not previously been generated by either player. We now argue that the generator achieves $\mu^{\el}_t(\mathsf{GCG}(E), g_m, t) \geq \alpha_m$ after finitely many rounds.
    
    Indeed, fix any round $t > 2T_m$. By time $T_m$, the generator can no longer receive credit for the first $T_m$ elements of $g_m$, since these elements are already late with respect to the deadline function $D(i)=i$. Moreover, by time $T_m$, the adversary has made at most $T_m$ outputs, and in the worst case these outputs may all be distinct elements of $(g_m)_t \setminus (g_m)_{T_m}$. Such elements can then no longer be generated by the generator for credit. Therefore, in the worst case, as many as $2T_m$ elements of $(g_m)_t$ are unavailable to the generator to generate for credit by time $T_m$, and hence at least $t-2T_m$ elements of $(g_m)_t$ remain available to be generated for credit at time $T_m$.
    
    From time $T_m$ onward, the generator greedily outputs the earliest available element of $g_m$, while the adversary can remove at most one such element per round. Thus, in the worst case, the adversary and generator alternate in generating elements from this set. Consequently, by time $t$, the generator generates at least $\lfloor \frac{t-2T_m}{2} \rfloor$ elements of $(g_m)_t$ on time. It follows that
    \[
        \mu^{\el}_t(\mathsf{GCG}(E), g_m, t) = \frac{1}{t} \sum_{j=1}^t \ind\{(g_m)_j \in \mathsf{GCG}(E)_j\} \geq \frac{\lfloor (t-2T_m)/2 \rfloor}{t}.
    \]
    Since $\alpha_m < \tfrac{1}{2}$ is fixed, there exists some finite $t_m$ such that for all $t \geq t_m$, $\lfloor (t-2T_m)/2 \rfloor/t \geq \alpha_m$. Therefore, stage $m$ terminates in finite time.
    
    Since each stage terminates in finite time and $K$ appears infinitely often in $G$, there exists an infinite sequence of stages $\{m_j\}_{j\in\bbN}$ such that $g_{m_j} = K$. At each such stage $m_j$, the generator achieves density at least $\alpha_{m_j}$ with respect to $K$. Finally, since $\alpha_{m_j}$ goes to $\tfrac{1}{2}$ as $j$ goes to infinity, it follows that there exists a sequence of indices for which $\mu^{\el}_i(\mathsf{GCG}(E),K,i)$ approaches $\tfrac{1}{2}$. Therefore, $\limsup_{i\to\infty} \mu^{\el}_i(\mathsf{GCG}(E),K,i) \geq \tfrac{1}{2}$.
    
    \textbf{$G$ is nonempty and stops growing after some finite time:} Suppose $G$ is nonempty and there exists some final stage $m^\star$ such that $g_{m^\star}$ is the last element ever appended to $G$. Let $T^\star$ denote the time at which $g_{m^\star}$ is appended to $G$. Since there are infinitely many rounds $t$ for which $\hat{L}^t = K$, and since $T^\star$ is finite, there must also be infinitely many rounds after time $T^\star$ for which this is true. Furthermore, since for all $t \geq T$ we have $\hat{L}^t \subseteq K$, it cannot be the case that $g_{m^\star} \subsetneq K$, as otherwise there would exist some later round $t' > T^\star$ for which $\hat{L}^{t'} = K$, implying $\hat{L}^{t'} \supsetneq \hat{L}^{t'-1}$ and hence that some additional language is appended to $G$, a contradiction. As such, we must have $g_{m^\star} = K$.
    
    By the argument above, each stage terminates in finite time. In particular, stages $1,\dots,m^\star-1$ all terminate in finite time, and so the generator eventually reaches stage $m^\star$. Since $g_{m^\star}$ is the final element of $G$, once the generator enters stage $m^\star$ it remains there forever.
    
    Let $T_{m^\star}$ denote the time at which the generator enters stage $m^\star$. From time $T_{m^\star}$ onward, the generator outputs the earliest element of $K$ that is both not yet late and has not previously been generated by either player. We now argue, exactly as in the case where $G$ remains empty, that this implies $\limsup_{i\to\infty}\mu^{\el}_i(\mathsf{GCG}(E),K,i)\geq \tfrac{1}{2}$.
    
    Fix any $i > 2T_{m^\star}$. By time $T_{m^\star}$, the generator can no longer receive credit for the first $T_{m^\star}$ elements of $K$, since these elements are already late with respect to the deadline function $D(i)=i$. Moreover, by time $T_{m^\star}$, the adversary has made at most $T_{m^\star}$ outputs, and in the worst case these outputs may all be distinct elements of $K_i \setminus K_{T_{m^\star}}$. Such elements can then no longer be generated by the generator for credit. Therefore, in the worst case, as many as $2T_{m^\star}$ elements of $K_i$ are unavailable to the generator to generate for credit by time $T_{m^\star}$, and hence at least $i-2T_{m^\star}$ elements of $K_i$ remain available to be generated for credit at time $T_{m^\star}$.
    
    From time $T_{m^\star}$ onward, the generator greedily outputs the earliest available element of $K$, while the adversary can remove at most one such element per round. Thus, in the worst case, the adversary and generator alternate in generating elements from this set. Consequently, by the time all elements of $K_i$ have been generated, the generator generates at least $\lfloor \tfrac{i-2T_{m^\star}}{2}\rfloor$ elements of $K_i$ on time. It follows that
    \[
        \mu^{\el}_i(\mathsf{GCG}(E),K,i) = \frac{1}{i}\sum_{j=1}^i \ind\{k_j \in \mathsf{GCG}(E)_j\} \geq \frac{\lfloor (i-2T_{m^\star})/2 \rfloor}{i}.
    \]
    Therefore, $\limsup_{i\to\infty}\mu^{\el}_i(\mathsf{GCG}(E),K,i)\geq \tfrac{1}{2}$.
    
    As we have now argued that $\limsup_{i\to\infty}\mu^{\el}_i(\mathsf{GCG}(E),K,i)\geq \tfrac{1}{2}$ in each case, we have proven the first claim.
\end{proof}

\end{document}